\documentclass[format=acmsmall, review=false]{acmart}
\pdfoutput=1
\usepackage{standalone}
\usepackage{acm-ec-22}
\usepackage{booktabs} 
\usepackage[ruled]{algorithm2e} 

\SetAlFnt{\small}
\SetAlCapFnt{\small}
\SetAlCapNameFnt{\small}
\SetAlCapHSkip{0pt}
\IncMargin{-\parindent}

\usepackage[utf8]{inputenc} 
\usepackage[T1]{fontenc}    
\usepackage{hyperref}       
\usepackage{url}            
\usepackage{booktabs}       
\usepackage{amsfonts}       
\usepackage{nicefrac}       
\usepackage{microtype}      
\usepackage{multicol}
\usepackage{dsfont}
\usepackage{float}
\usepackage[]{bm}
\usepackage[normalem]{ulem}

\usepackage{amsmath, amsthm, bbm, cancel}
\usepackage{subcaption}
\usepackage{tikz}
\usepackage[capitalise]{cleveref}
\usepackage{xparse}
\usepackage[]{caption}
\usepackage[]{subcaption}

\usepackage[textsize=tiny]{todonotes}

\newtheorem{theorem}{Theorem}[section]
\newtheorem{lemma}[theorem]{Lemma}

\newtheorem{proposition}[theorem]{Proposition}

\newtheorem{assumption}[theorem]{Assumption}

\NewDocumentEnvironment{myproof}{o}
  {\IfNoValueTF{#1}{\paragraph{{\textit{Proof.}}}} {\paragraph{{#1.} }} }
  {\hfill$\square$}

\setcitestyle{authoryear}

\title{Synthetically Controlled Bandits}

\author{Vivek Farias}
\email{vivekf@mit.edu}
\affiliation{
  \institution{Sloan School of Management, Massachusetts Institute of Technology}
  \city{Cambridge}
  \state{Massachusetts}
  \country{USA}
}

\author{Ciamac Moallemi}
\affiliation{
  \institution{Graduate School of Business, Columbia University}
  \city{New York}
  \state{New York}
  \country{USA}
}

\author{Tianyi Peng}
\affiliation{
  \institution{Department of Aeronautics and Astronautics, Massachusetts Institute of Technology}
  \city{Cambridge}
  \state{Massachusetts}
  \country{USA}
}
\author{Andrew Zheng}
\affiliation{
  \institution{Operations Research Center, Massachusetts Institute of Technology}
  \city{Cambridge}
  \state{Massachusetts}
  \country{USA}
}

\begin{abstract}
This paper presents a new dynamic approach to experiment design in settings where, due to interference or other concerns, experimental units are coarse. `Region-split' experiments on online platforms are one example of such a setting. The cost, or regret, of experimentation is a natural concern here. Our new design, dubbed Synthetically Controlled Thompson Sampling (SCTS), minimizes the regret associated with experimentation at no practically meaningful loss to inferential ability. We provide theoretical guarantees characterizing the near-optimal regret of our approach, and the error rates achieved by the corresponding treatment effect estimator. Experiments on synthetic and real world data highlight the merits of our approach relative to both fixed and `switchback' designs common to such experimental settings.
\end{abstract}

\begin{document}

\begin{titlepage}

\maketitle

\end{titlepage}

\section{Introduction}

Experimentation is a crucial tool deployed in the data-driven improvement of modern commerce platforms. On such platforms, it is often the case that a new product feature or algorithmic tweak is broadly rolled out only after its prospective benefit is understood via an appropriately designed experiment. In some cases, an appropriate unit of experimentation is simply an end user. In such cases, experiment design and inference, while not entirely trivial, is relatively well understood. On the other hand, it is often the case that the intervention, whose effect the experiment seeks to characterize, induces interactions among individual users of the platform. Often referred to as `interference', this effectively violates the Stable Unit Treatment Value Assumption (SUTVA) that is typically assumed in most designs, and necessary for correct inference. There is an emergent and exciting literature focused on experiment design and inference in the presence of interference. 

It remains unclear how to robustly characterize the bias induced by interference. As such, a common strategy used to obviate interference concerns in practice, is simply to pick a unit of experimentation that is sufficiently coarse. For instance, a ride hailing platform experimenting with a new payment feature would simply choose the unit of experimentation to be a region or city. 
Since the pool of such coarse experimental units is by definition smaller, picking the appropriate controls is no longer a simple matter. In addition, the counterfactual value of the quantity being measured for any unit is likely to have temporal effects, and may potentially even be non-stationary so that even the natural controls for so-called `switch-back' designs are insufficient. These very challenges arise in program evaluation, a common task in empirical economics. There, the synthetic control methodology is seen as `arguably the most important innovation ... in the last 15 years' \cite{athey2017state}. This approach seeks to construct a `synthetic' control via a linear combination of non-treatment units that best approximates the treatment unit prior to the treatment period. While originally intended primarily for inference given observational data, the synthetic control approach has become a go-to approach for experiment design and inference on platforms in settings where the unit of experimentation is coarse.
The only drawback to this overall scheme is the {\em cost of experimentation}: any `regret' from an undesirable intervention is now borne at the level of a city or region (as opposed to a substantially smaller group of users) over the period of the experiment.  

This paper proposes a new approach to learning in settings such as the `region-split' experiments described above enabled by a novel device: the {\em synthetically controlled bandit}. This approach (a) yields near-optimal sub-linear regret {\em and} (b) recovers the treatment effect at the same rate as a traditional experiment with synthetic controls on the event that the treatment effect is positive. On the event that the treatment effect is negative, the approach simply learns that this is the case but does not recover a precise estimate of the treatment effect. As such, the synthetically controlled bandit largely eliminates the cost of experimentation at the expense of being able to learn the treatment effect precisely only on the event where this treatment effect is positive. Since in practical applications, a precise estimate of the treatment effect is only of value when this treatment effect is positive (so as to facilitate, for instance, cost-benefit analyses for a roll-out of the intervention), this new approach makes possible an attractive tradeoff.

\subsection{The Synthetically Controlled Bandit}

{\em Minimizing the Cost of Exploration: }Consider a setting where the decision whether or not to treat the treated unit (city, region, etc.) in any given epoch is a dynamic one. Over some experimentation horizon, a natural goal aligned with minimizing the cost of experimentation, would be to minimize {\em regret}. That is, over the experimentation horizon, we effectively minimize the expected number of times a sub-optimal treatment option was chosen for the treated unit. It turns out that in the synthetic control setting, this problem is equivalent to a linear contextual bandit, wherein the context at each period is a low-dimensional latent vector (the so-called `unobserved common factors' in the corresponding synthetic control model). This bandit problem has several salient features worth noting:
\begin{enumerate}
\item Since observations across all units are made contemporaneously, the context vector is unavailable at the time of decision making. 
\item We never observe historical context vectors directly; instead we only ever observe an unknown, noisy, linear transformation of these vectors. 
\item This unknown linear transformation can never be recovered exactly even with infinite data: instead, we can only hope to recover it up to a rotation. 
\end{enumerate}
Put succinctly, the underlying contextual bandit is one where the context is not available at the
time of decision making, and can, post-facto, only be recovered up-to some unknown rotation and
noise. Our primary technical contribution is an algorithm that, despite the challenges above,
achieves a regret that scales like $r \sqrt{T}$. Here $r$ is the dimension of the latent context
and $T$ is the experimentation horizon. We dub our approach {\em Synthetically Controlled Thompson Sampling} (SCTS).
SCTS consists of a Thompson sampling routine with
carefully designed `exploration noise'.  Contexts are recovered via principal components analysis
(PCA) on historical observations. Importantly, our sampler is robust to the errors in context
recovery due to noise and the inability to recover rotations. The careful design of our sampler
also allows for a linear dependence on the latent context dimension $r$ as opposed to the
$r^{3/2}$ achieved by state-of-the-art Thompson sampling approaches; a result of independent
interest.
\newline
\newline
{\em Inference: }The vanilla synthetic control estimator is not asymptotically normal, and as such, inference is either via high-probability confidence intervals, or else general purpose non-parametric approaches (such as the bootstrap or permutation tests). We propose a treatment effect estimator and establish high probability confidence intervals for this estimator that (a) on the event that the true treatment effect is positive line up precisely with the vanilla synthetic control confidence intervals and (b) on the event that the true treatment effect is negative will simply be the negative half-line. In other words, we provide the same quality of inference as in a synthetic control experiment on the event the the treatment effect is positive, but when the effect is negative, are only able to detect that this is the case. Since quantification of the treatment effect is typically only relevant when the effect is positive (so as to facilitate, for instance, cost-benefit analyses for a roll-out of the intervention), it is unclear that the loss in inferential capabilities relative to synthetic control has practical consequences. With an eye to practice, we propose the use of re-randomization based hypothesis tests and confidence intervals derived from inverting these tests. We see, on real world data, that these tests are highly powered even for small treatment effects and that the confidence intervals derived from them provide near-ideal coverage.   
\newline
\newline
{\em Computational Experience: }We present experimental work on both a synthetic data setup (wherein the synthetic control model holds by construction), as well as on real world data, where this setup is, at best, an approximation. In both cases, we see that SCTS employs a sub-optimal intervention for a negligible fraction (typically a single digit percentage) of epochs over the experimentation horizon. Compared with both a fixed and switchback design, SCTS thus materially reduces the cost of exploration. Despite this, we see that our treatment effect estimator correctly identifies whether or not the treatment effect is positive in every single one of our instances. Importantly, on the instances where the treatment effect is positive, the relative RMSE of our estimator is comparable to state-of-the-art estimators for both the switchback and fixed designs, and in fact outperforms these incumbents in the real-data setting. Finally, as discussed above, re-randomization based hypothesis tests and confidence intervals provide near-ideal coverage, even on real-world data, allowing for effective inference.

\subsection{Related Literature}
\label{sec:literature}

{\em Synthetic Control and Inference: }The notion of synthetic control was introduced initially in the context of program evaluation: \cite{abadie2003economic,abadie2010synthetic} are seminal papers that propose to recover the counterfactual in an observational setting by creating a ``synthetic control''. Specifically, they proposed constructing a convex combination of control units that matches the treated unit in pre-treatment periods. A series of follow-on studies proposed distinct estimators by employing different constraints and regularizers (e.g, \cite{hsiao2012panel,doudchenko2016balancing,li2017estimation,arkhangelsky2019synthetic, ben2021augmented}). See \cite{abadie2019using} for a review of this vibrant literature. Since the underlying generative model justifying the synthetic control framework is in fact a factor model, it is natural to consider using PCA-like techniques in the recovery of a synthetic control; the present work leverages such techniques in a dynamic context. This approach is especially relevant in the setting where the size of the `donor pool' is large. \cite{athey2021matrix, xu2017generalized, amjad2018robust, bai2019matrix, amjad2019mrsc,agarwal2021robustness,farias2021learning} are all papers in this vein. \cite{farias2021learning} in particular compute a min-max optimal estimator for a generalization of the synthetic control problem. It is indeed possible to compute limiting distributions for synthetic control estimators in various special cases. For instance, if one were willing to make probabilistic assumptions on the data, it is possible to compute a limiting distribution for the synthetic control estimator (roughly, this distribution is a projection of the OLS limiting distribution to a convex set); \cite{li2020statistical}. Several of the references above also compute limiting distributions in special cases of the problem. In practice, however, inference for average treatment effects in synthetic control is done via non-parametric methods such as permutation tests;  see \cite{chernozhukov2021exact}.

\noindent {\em Synthetic Control in Experiment Design for Commerce: }Not surprisingly, synthetic control approaches have gained traction in modern commerce settings; as a relatively early example \cite{brodersen2015inferring} describes an approach and corresponding software used by Google in the context of marketing attribution. Going further, however, synthetic control has come to be viewed as an important tool in experiment design as well, as opposed to simply in observational settings. For instance, \cite{Chen2020Lyft,uberSynthetic2019} describe practical designs, at Lyft and Uber respectively, that assume a synthetic control model holds across units. In a theoretical direction, \cite{doudchenko2019designing} and \cite{abadie2021synthetic} consider the problem of how best to select an experimental unit assuming the synthetic control model holds, motivated by problems at Facebook and the sorts of `region-split' experiments common to ride-sharing platforms respectively. Like this work, the present paper also actively uses the synthetic control model in experiment design; in our case these designs are `dynamic'.

\noindent {\em Contextual Bandits and Inference: }The dynamic design that this paper constructs is, in a certain idealized sense, a linear contextual bandit. This sort of bandit is classical, studied at least as early as \cite{auerUsingConfidenceBounds2002}. The practical constraints around our design necessitate a sampling methodology that draws from recent work on Thomson sampling for such bandits; see \cite{agrawalThompsonSamplingContextual2013,abeilleLinearThompsonSampling2017,kvetonMetaThompsonSampling2021}. As noted elsewhere, the fundamental challenge we must address is that we never observe contexts directly. There is some limited work discussing the use of dynamic bandit based designs in clinical trials, \cite{villarMultiarmedBanditModels2015, berry2012adaptive}. 

Turning to inference, it is well known that naive sample estimates of arm means in bandits are biased (see e.g. \cite{villarMultiarmedBanditModels2015,nieWhyAdaptivelyCollected2018}). A recent line of work considers `post-contextual bandit' inference, where certain importance-weighted estimators are shown to be unbiased and asymptotically normal; see \cite{hadadConfidenceIntervalsPolicy2021, bibautPostContextualBanditInference2021} and also \cite{deshpandeAccurateInferenceAdaptive2018} for a different approach to the problem under a linear reward model. Extending these to our setting is an exciting direction for future work. In addition to the observability of the contexts, such an extension will also need to address the general problem that this line of work requires a type of forced exploration of arms which may not be consistent with our bandit algorithm. The present paper simply constructs high probability confidence intervals via the usual self-normalized martingale concentration bounds. While loose in practice, they already illustrate that we can expect rates that are essentially on par with what is possible for the vanilla synthetic control estimator. In our experimental work, we complement these with bootstrapped confidence intervals and permutation tests for significance that we show work adequately.


\section{Model}
\label{sec:model}

We measure some quantity of interest for an experimental unit (e.g., a region, in a region-split
experiment) over a pre-treatment period of length $T_0$, and a subsequent treatment period of
length $T$. We denote the measurement made on this experimental unit at any epoch $t$ by
$y^{0}_t \in \mathbb{R}$.  We assume that the pre-treatment period consists of
epochs in $\{-T_0+1,\dots,0\}$ and that the treatment period consists of epochs in
$\{1,2,\dots,T\}$. We denote by $a_t \in \{0,1\}$ the indicator of whether or not the experimental
unit is treated at time $t$, so that $a_t = 0$ for all epochs in the pre-treatment period. We
assume that $y^0_t$ is determined by the structural equation
\begin{equation}
  \label{eq:reward-model}
  y^0_t = \tau^* a_t + \langle{\lambda^*},{\bar z_t }\rangle + \epsilon^{0}_t,
%
\end{equation}
where $\tau^* \in \mathbb{R}$ is an unknown treatment effect, $\bar z_t \in \mathbb{R}^r$ is an (unknown) set of $r$ `shared common factors' and $\lambda^* \in \mathbb{R}^r$ is a set of (unknown) `factor loadings' specific to the experimental unit. The noise $\epsilon^{0}_t$ is assumed to be independent Gaussian with mean zero, and standard deviation $\sigma$. The synthetic control paradigm assumes a generative setting where a weighted combination of observations in the pool of donor units closely approximates counterfactual observations on the experimental unit. Specifically we assume a pool of $n \geq r$ donor units, where for the $i$th such unit
\begin{equation}
  \label{eq:donor-model}
y^i_t = \langle{\lambda^{i}}, \bar z_t \rangle + \epsilon^{i}_t.
\end{equation}
Here $\bar z_t$ is the same set of shared common factors and $\lambda^{i} \in \mathbb{R}^r$ is a set of factor loadings specific to the $i$th donor unit. As before, $\epsilon^{i}_t \sim \mathcal{N}(0,\sigma^2)$.





\subsection{Dynamic Design and Estimator}

The typical synthetic control design is simply to set $a_t = 1$ in the treatment period. Now define the cost of experimentation incurred at time $t$ by the `regret' incurred in that epoch,
\[
R_t \triangleq |\tau^*| \cdot \bigl(\mathbf{1}\left\{\tau^* < 0 \right\}a_t + \mathbf{1}\left\{\tau^* \geq 0 \right\}(1-a_t) \bigr).
\]
This definition captures both the negative impact of a sub-optimal treatment, and the opportunity cost of not using the treatment should it be optimal. The total cost incurred by the typical synthetic control design may then scale linearly with the length of the treatment period, $T$.

In the interest of minimizing the total cost of experimentation, we allow $a_t$ to be
dynamic. Specifically, require $a_t$ to be selected according to a randomized policy that is adapted to
$\mathcal{F}_t \triangleq
\sigma(
(a_s, y^0_s, y^1_s, \dots,y^n_s), s < t
)
$, the filtration generated by all treatment decisions taken, and observations made in the
treatment and donor units, up to time $t-1$.
The total expected cost of experimentation, $R(T)$,
or total expected regret incurred over the course of the experiment is then simply $\mathbb{E}
\left[ \sum_{t=1}^T R_t \right],$ where the expectation is over the noise in observations and randomization in the design. Finally an estimator of the treatment effect, $\hat \tau$, is simply an $\mathcal{F}_{T+1}$ measurable random variable that ideally provides a good approximation to $\tau^*$.
With this setup, we are now able to state the problems we wish to address:

\begin{itemize}
\item First, we would like to produce a dynamic design, i.e., a process $\{a_t\}$, that minimizes the total cost of experimentation $R(T)$.
\item Second, on the inferential side, we would like to design an estimator $\hat \tau$ for which
  $|\hat \tau - \tau^*|$ is `small' with high probability, particularly when $\tau^* > 0$, since
  quantification of the treatment effect is most important in this case.
\end{itemize}
In what follows, we describe our main results, making precise the trade-off we achieve between controlling $R(T)$ and the quality of our estimator.

\subsection{Results}

We propose a dynamic design, we dub {\em Synthetically Controlled Thompson Sampling} (SCTS), which we show achieves near-optimal experimentation cost:
\begin{theorem}[Informal]: \label{thm:informal-main} Assume that the number of donor units
  $n = \Omega(T)$. Then, under mild assumptions on the shared common factors, we have that SCTS
  incurs a cost of experimentation,
  $R(T) = O\left(r \sqrt{T} \log(T)\right)$.
\end{theorem}
In a nutshell, SCTS essentially eliminates the cost of experimentation, which as we noted earlier
will scale linearly with $T$ in the traditional fixed synthetic control design. It is also worth
placing the precise regret guarantee in context. To that end, note that we have no information
pertaining to $\bar z_t$ (which is essentially arbitrary in the synthetic control model) at the
time we decide on $a_t$. Imagine for a moment, however, that at time $t$ we observed $\bar z_s$
for all $s < t$. Treating this as a two-armed linear contextual bandit, Thompson sampling applied to this setup is then known to achieve
$\tilde O(r^{3/2} \sqrt{T})$ regret
\cite{abeilleLinearThompsonSampling2017}.\footnote{The $\tilde O(\cdot)$ notation suppresses dependence on logarithmic factors, $\log(T)$ and $\log(r)$. }
Of course, even the history of the shared common factors is
not available; rather these must be inferred from our observations over the donor units. As a further complication, even with noiseless observations of $y^i_t$ on the donor units we would only
succeed in recovering the common factors $\bar z_t$ up to a rotation. In light of these salient
problem features it is notable that our regret guarantee depends linearly on $r$, which is typically
much smaller than the ambient number of donor units, $n$. This guarantee is our main theoretical
result.

We turn next to inference. There we know that the nominal \cite{abadie2010synthetic} synthetic
control estimator $\tau^{\rm SC}$ achieves with probability $1-O(\delta)$,
\[
|\tau^{\rm SC} - \tau^*|
\lesssim
\frac{\sigma}{\sqrt{T}}
\sqrt{\log(1/\delta)}
+
\frac{c_2\sigma}{c_1 \sqrt{T_0}}
\sqrt{\log(1/\delta) + \log n}.
\]
The regularization implicit in this estimator achieves a rate that is largely independent of $n$.\footnote{We write $A \lesssim B$ if $A \leq cB$ for some absolute constant $c$.} The constant $c_1$ depends on $r$ whereas the constant $c_2$ depends on the size of the shared common factors. Asymptotic distributions for this estimator without further distributional assumptions on the common factor process are unknown.


Our own estimator, $\tau^{\rm SCTS}$, works as follows. We compute $\tau^{\rm SC}$ (i.e., the
vanilla synthetic control estimator) using only observations in the pre-treatment period and those
experimental epochs over which $a_t = 1$. We set $\tau^{\rm SCTS} = \tau^{\rm SC}$ on the event
that the intervention was used over at least $T/2$ epochs; otherwise we set $\tau^{\rm SCTS} =
0$. We are then able to show that when $\tau^* > 0$, with probability $1-O(\delta)- \tilde{O}(1/\sqrt{T})$,
\[
|\tau^{\rm SCTS} - \tau^*|
\lesssim
\frac{\sigma}{\sqrt{T}}
\sqrt{\log(1/\delta)}
+
\frac{c_2\sigma}{c_1 \sqrt{T_0}}
\sqrt{\log(1/\delta) + \log n}.
\]
On the other hand, when $\tau^* \leq 0$, with probability $1-\tilde{O}(1/\sqrt{T})$, $\tau^{\rm SCTS} = 0$.
%
Contrasting this with the high probability confidence intervals for $\tau^{\rm SC}$, we see that on the event that $\tau^* > 0$, we get the same intervals as the synthetic control estimator. On the event that $\tau^* < 0$, all we learn is that the treatment effect is negative which is in essence the price we pay for controlling the cost of experimentation. The result follows from a simple idea expanded on in Section~\ref{sec:inference}. In our computational experiments, we see that re-randomization tests for p-values and the corresponding inverted hypothesis tests~\cite{fisher1966design} for confidence intervals provide adequate power and coverage.


In their totality, these results show that we can largely eliminate the cost of experimentation at a modest cost to inference: when the treatment effect is negative we only learn that this is the case with high probability, as opposed to getting a precise estimate of the effect. Since in practical settings a precise estimate of the treatment effect is typically only needed when the treatment effect is positive (so as to ascertain whether the cost of implementing the intervention is justified), this is perhaps a modest price to pay.

%


\section{Synthetically Controlled Thompson Sampling}
\label{sec:orgfae4997}

We introduce SCTS, which adapts Thompson Sampling (TS) to the problem of dynamically selecting
interventions so as to minimize expected regret in the setting of the previous section. The
algorithm is conceptually simple: at the start of each epoch, $t+1$, we compute a distribution
$\mathcal{D}_t^{\rm TS}$ over `plausible' values of $\tau^*$. This distribution may be thought of
informally as an approximation to a posterior over $\tau^*$ under a non-informative prior, given
the information available up to and including time $t$. We then sample from this distribution, and pick $a_{t+1}=1$ if and only if the sampled value, $\tilde \tau_t$ is non-negative. To construct $\mathcal{D}_t^{\rm TS}$, we
\begin{enumerate}
\item First, estimate the (unobserved) shared common factors $\bar z_s$ for $s \leq t$. We will accomplish this via PCA.
\item Plugging-in the estimates obtained for $\bar z_s$ above into the structural equation \eqref{eq:reward-model}, we compute an estimate of $\tau^*$ and $\lambda^*$ via ridge regression.
\item We use the estimates of $\tau^*$ and the precision matrix obtained from the regression in the previous step to construct our approximation to the posterior on $\tau^*$, $\mathcal{D}_t^{\rm TS}$.
\end{enumerate}
Next, we make precise each of these steps, assuming, simply for notational convenience, that $T_0 = 0.$
\newline
\newline
\noindent \textbf{Estimating Shared Common Factors: }Recall from \eqref{eq:donor-model}, that for each donor unit $i$ and epoch $s$, we observe
$
y^i_s = \langle{\lambda^{i}}, \bar z_s \rangle + \epsilon^{i}_s.
$
Define by $Y_t \in \mathbb{R}^{n \times t}$ the matrix with $(i,s)$ entry $y^i_s$, and similarly, define by $E_t \in \mathbb{R}^{n \times t}$ the noise matrix with $(i,s)$ entry $\epsilon^i_s$.
Now, let $\Lambda \in \mathbb{R}^{n \times r}$ be the factor loadings matrix with $i$th row $\lambda^{i\top}$, and denote by $\bar Z_t \in \mathbb{R}^{t \times r}$ the common factors matrix with $s$th row $\bar z_s^\top$. By \eqref{eq:donor-model}, we then observe at time $t+1$:
\begin{equation}
  Y_t
  =
  \Lambda\bar{Z}^\top_t  + E_t. \label{eq:matrix-model-control}
\end{equation}
We estimate $\bar Z_{t}$ at time $t+1$ by solving
\begin{equation}
        \label{eq:latent-estimation}
        \min_{Z \in \mathbb{R}^{t \times r},\ \Lambda \in \mathbb{R}^{n \times r}}\ \left\|Y_{t} - \Lambda Z^\top \right\|^2.
\end{equation}
We fix a specific solution to the above optimization problem via PCA. Specifically, let
$Y_t = \hat{U}_t \hat{\Sigma}_t \hat{V}_t^\top$ be any singular value decomposition (SVD) of $Y_t$. Denote by $\hat{U}^r_{t}$ and
$\hat{V}^r_{t}$ the matrices obtained from the first $r$ columns of $\hat{U}_t$ and $\hat{V}_t$
respectively. Finally, let $\hat{\Sigma}^r_t$ be the sub-matrix obtained from $\hat \Sigma_t$ from
its first $r$ rows and columns. By the Young-Eckart theorem, an optimal solution to
\eqref{eq:latent-estimation}, $(\hat \Lambda_t, \hat Z_t)$, can be obtained by setting
$\hat \Lambda_t \triangleq \sqrt{n} \hat U_t^r$, and
\[
\hat Z_t \triangleq \frac{1}{\sqrt{n}} \hat V^r_t \hat \Sigma^r_t.
\]
We recognize $\hat Z_t$ as precisely the usual `PCA loadings'; $\hat Z_t$ will serve as our approximation to $\bar Z_t$.
\newline
\newline
\noindent \textbf{Ridge Regression: } Recall that in our synthetic control model, we have for the treatment unit, $y^0_s = \tau^* a_s + \langle{\lambda^*}, \bar z_s\rangle + \epsilon^{0}_s$ at each epoch $s$. At time $t+1$, we employ this structural equation to estimate $\tau^*$ via least squares, using as a plug-in estimator\footnote{While the subscript $t$ in $\hat z_{s,t}$ makes precise that this is our estimate of $\bar z_s$ at time $t$, we will sometimes drop the $t$ subscript when clear from context.} for $\bar z_s^\top$, $\hat z_{s,t}^\top$, the $s$th row of $\hat Z_t$. Our estimate of $\tau^*$ at time $t+1$, $\hat \tau_t$, is obtained as the solution to the regularized least squares problem:
\begin{equation}
  \label{eq:rls}
      \min_{\tau \in \mathbb{R},\ \lambda \in \mathbb{R}^{r}}\ %
       \sum_{s \leq t}
      \left(y^0_s - \tau a_s - \langle \lambda, \hat z_{s,t} \rangle \right)^2
      + \rho (\tau^2 + \|\lambda\|_{2}^2).
\end{equation}
Here, $\rho > 0$ is a regularization penalty.\footnote{We fix $\rho \triangleq 1$ throughout the paper.}

We find it convenient to define the `precision matrix'
$\Omega_t \in \mathbb{R}^{(r+1) \times (r+1)}$ of the estimator
$\hat\theta_t^\top \triangleq [\hat \tau_t \ \hat \lambda^\top_t]$. Specifically, if we
denote $x_{s,t}^\top \triangleq [a_s \ \hat z_{s,t}^\top]$, then
$\Omega_t \triangleq \rho I + \sum_{s \leq t} x_{s,t}x_{s,t}^\top$. The `variance'\footnote{While for
  expositional purposes we use the terminology `precision matrix' and `variance', these quantities
  are of course not a precision matrix or variance since the design of the regression problem is
  not fixed.} of our estimator $\hat \tau_t$ is simply
$(\Omega_t)^{-1}_{1,1} \triangleq \hat \sigma^2_t$.  \newline \newline
\noindent \textbf{Approximate Posterior: } For our approximation to the posterior on $\tau^*$ at
time $t+1$, we take $\mathcal{D}_t^{\rm TS}$ to be the uniform distribution,
$\mathrm{Unif}[\hat\tau_t - \beta_t \hat\sigma_t, \hat \tau_t + \beta_t \hat \sigma_t]$.  Here
$\beta_t > 0$ is a time-dependent `expansion' factor we make precise later; for now we may simply
consider $\beta_t$ to be an increasing sequence with
$\beta_t = O(\sqrt{r \log (r t)})$.
\newline
\newline
\noindent
As discussed earlier, SCTS draws a sample, $\tilde \tau_t$ from $\mathcal{D}_t^{\rm TS}$ at time $t+1$. Then, SCTS sets $a_{t+1} =1$ if and only if $\tilde \tau_t \geq 0$.

\subsection{Discussion: Inconsistent Designs, the Failure of Optimism}\label{sec:challenges}

\textbf{\textsf{Inconsistent Designs: }}Notice that the design employed in the regression~\eqref{eq:rls} is {\em inconsistent} from period to period in the sense that the estimate for any fixed context $\bar z_s$ in the design matrix changes from period to period. Part of this is simply due to noise -- as time goes on we hope to compute a more accurate estimate of $\bar z_s$ for any fixed $s$. However, as it turns out even in the absence of noise (i.e., if $E_t$ were identically zero), we would still not expect consistency in the design since even in that case, we would only ever be able to recover the contexts up to a rotation. A priori it is unclear whether this inconsistency will allow for effective recovery of the treatment effect, and as such it is unclear whether we can expect the algorithm we have described to achieve low regret.
\newline
\newline
\noindent
\textbf{\textsf{Optimistic Algorithms: }}A natural upper confidence bound (UCB) style alternative to the algorithm we have described, might proceed by defining the upper confidence bound ${\rm UCB}_t \triangleq \hat \tau_t + \beta_t \hat \sigma_t$, and then setting $a_{t+1} =1$ if and only if ${\rm UCB}_t \geq 0$. Perhaps surprisingly, this algorithm would incur {\em linear} regret in general; see Appendix~\ref{sec:fail-UCB} that provides a simple and decidedly non-pathological example of this phenomenon. It is thus interesting that the `sampling' aspect of the algorithm above eventually plays a crucial role in achieving sub-linear regret.

\section{Regret Analysis}\label{sec:regret}
This section provides a regret analysis for SCTS. We begin by restating Theorem~\ref{thm:informal-main} formally. In order to do so, we must first state our assumptions, which concern the expected value of the observation matrix on the donor units, i.e., $\mathbb{E} Y_T \triangleq \bar Y_T$, a rank $r$ matrix. Specifically, we make assumptions on the decomposition $\bar Y_T = \Lambda \bar Z^\top_T$. To do so, we first note that in our model, it is possible to assume, without loss, a canonical version of this decomposition (note that the selection of $\Lambda$ and $\bar Z_{T}$ is not unique due to the free
choice of a rotation). In particular, letting $\bar Y_T = \bar U \bar \Sigma \bar V^\top$ be an SVD
of $\bar Y_T$, we may assume without loss, that $\Lambda = \sqrt{n} \bar U$ and
$\bar Z_{T} =  \bar V \bar \Sigma /\sqrt{n}$; see \cref{sec:canonical-decomp} for details. Given
this canonical decomposition, we assume:

%
%
%

\begin{assumption}
For all $t$, $\|\bar{z}_t\|_{2}$ is upper bounded by a constant, $B$, and $\|\lambda^*\|_{2} = O(\sqrt{r})$.
\label{ass:bar_y_decomp}
\end{assumption}

\noindent We can now state our main regret bound for SCTS.

\begin{theorem}[SCTS Regret]
Let $n = \Omega(T)$. Then, under Assumption~\ref{ass:bar_y_decomp}, SCTS achieves expected regret $R(T) = O(r \sqrt{T}\log(T))$.
\label{thm:regret}
\end{theorem}

The $O(\cdot)$ notation in the above regret bound ignores terms that depend polynomially on $\sigma$ and $B$. As stated earlier, the linear dependence on $r$ above is of note. We also observe that beyond its rank, our guarantee remarkably has no further dependence on the spectrum of $\bar Y_T$.

\subsection{Proof Architecture for Theorem~\ref{thm:regret}}

The proof of Theorem~\ref{thm:regret} follows a familiar architecture that decomposes regret over
time. We lay out this architecture here and will make precise two key results
(Propositions~\ref{thm:clean-prob} and \ref{thm:single-step-regret}) that enable the
proof. Establishing these propositions is the core challenge in establishing a useful regret
guarantee.  In what follows, we find it convenient to define the `true context' vector
$\bar{x}^\top_t \triangleq [a_t \ \bar{z}^\top_t]$, and the associated precision matrix
$\bar{\Omega}_t \triangleq \rho I + \sum_{s=1}^{t} \bar{x}_{s} \bar{x}_{s}^\top$. The Elliptical Potential
Lemma \cite{abbasi-yadkoriImprovedAlgorithmsLinear2011} then states

\begin{lemma}[Elliptical Potential Lemma] Under Assumption~\ref{ass:bar_y_decomp}, it holds that
$$
\sum_{t=0}^{T-1} \left\|\bar{x}_{t+1}\right\|_{\bar{\Omega}_t ^{-1}} = O \left(\sqrt{r T \log T} \right).
$$
\label{thm:epl}
\end{lemma}

Now, we must control the error in our estimates of the context vectors, $\bar Z_t$, and the consequent error in our estimation of $\tau^*$. Specifically, define the event that the error in recovering $\bar Z_t$ is small, $C_t^{\rm latent}$, according to
\[
C_t^{\rm latent} \triangleq \left\{
\inf_{\Phi \in \mathcal{O}_r}
\|\bar{Z}_{t} - \hat{Z}_{t} \Phi \| \leq \alpha
\right\}.
\]
Here $\mathcal{O}_r$ is the set of $r$-dimensional rotations, and $\alpha \leq c \sigma$ for some universal constant $c$. Observe that we only control this error up to a rotation. We define the event that the error in estimating $\tau^*_t$ is small, $C_t^{\rm est}$, according to
\[
C_t^{\rm est} \triangleq \left\{
|\tau^* - \hat{\tau}_t| \leq \beta_t \hat \sigma_t / 2
\right\}.
\]
We will control single-step regret on the `clean' event that both these errors are controlled,
$C_t \triangleq C_t^{\rm latent} \cap C_t^{\rm est}$; this is a high probability event:
\begin{proposition}[clean event] \label{prop:clean-event}
For all $t$, under Assumption~\ref{ass:bar_y_decomp},
$
\mathbb{P}\left(C_t\right) \geq 1 - O(1/t^2).
$
\label{thm:clean-prob}
\end{proposition}
This result is proved in Appendix \ref{sec:proof-clean-event}. The result relies on an analysis generalizing 
the Davis-Kahan theorem (to control $C_t^{\rm latent}$) and the usual self-normalized martingale
concentration bounds (to control $C_t^{\rm est}$). A key additional ingredient is needed --- in
controlling $C_t^{\rm est}$, we must deal with the issue of inconsistent designs in the
regression~\eqref{eq:rls}. As discussed in \cref{sec:challenges}, one issue driving this
inconsistency is the fact that $\bar Z_t$ can only be recovered up to a rotation. The proof of
Proposition~\ref{thm:clean-prob} overcomes this challenge by showing that the actions selected
under distinct rotations are in fact equal in distribution so that we can assume a canonical
rotation without loss of generality.
We now state our bound on single-step regret; this is the key result that enables our regret analysis and will be proved later in this section:
\begin{proposition}[Single-step regret] For some universal constant $c_1$, we have for all $t$,
$$
\mathbb{E}\left[R_{t+1} \mid C_t\right] \leq c_1(1 + \alpha) \beta_t \mathbb{E}\left[\|\bar{x}_{t+1}\|_{\bar{\Omega}_t^{-1}} \big| C_t\right].
$$
\label{thm:single-step-regret}
\end{proposition}
Finally, Theorem \ref{thm:regret} follows from summing single-step regret and applying the Elliptical Potential Lemma,
\begin{align*}
\mathbb{E}\left[  \sum_{t=0}^{T-1} R_{t+1}\right]
&
\leq \sum_{t=0}^{T-1}  \left(1 - \mathbb{P}\left(C_t \right)\right) |\tau^*|
+
\sum_{t=0}^{T-1} \mathbb{P}\left(C_t \right)\mathbb{E}\left[R_{t+1} \mid C_t\right]
\\
&
\leq
 O(1) + \sum_{t=0}^{T-1} \mathbb{P}\left(C_t \right)\mathbb{E}\left[R_{t+1} \mid C_t\right]
 \\
&
\leq
O(1) +
c_1(1 + \alpha) \beta_{T-1}
\sum_{t=0}^{T-1}
\mathbb{E}\left[\|\bar{x}_{t+1}\|_{\bar{\Omega}_t ^{-1}} \right]
\\
&
\leq
O(1) + c_1(1 + \alpha) \beta_{T-1}  O \left(\sqrt{r T \log T} \right) = O \left(r \sqrt{T} \log(T)\right).
\end{align*}
The second inequality above relies on the fact that the clean event occurs with high probability (Proposition~\ref{prop:clean-event}). The third inequality relies on our bound on the expected single-step regret, Proposition~\ref{thm:single-step-regret}, and the final inequality is simply the elliptical potential lemma.

\subsection{Bounding the single-step regret}

Proposition~\ref{thm:single-step-regret} is a critical enabler of our regret guarantee. We now proceed with that proof. We will begin with stating three lemmas key to the proof. To that end, let $\bar x_{t+1}^{* \top} \triangleq [a^* \ \bar z_{t+1}^\top]$, where $a^* = 1$ if $\tau^* > 0$ and $a^* = 0$ otherwise; and recall that $\bar x_{t+1}^\top \triangleq [a_{t+1} \ \bar z_{t+1}^\top]$. Then we have:

\begin{lemma}
\label{le:single_step_ub1}
On $C_t^{\rm est}$,
$
R_{t+1} \leq
2\beta_t \left(
\|\bar x_{t+1}^* \|_{\Omega_t^{-1}}
+
\|\bar x_{t+1} \|_{\Omega_t^{-1}}
\right)
$.
\end{lemma}
This lemma is crucial to connecting single-step regret with an appropriate norm of $\bar x_t$ so as
to eventually facilitate the use of the elliptical potential lemma and will be proved later in
this section. It is interesting to note that~\cite{abeilleLinearThompsonSampling2017} proves a
version of this result which in our setting would eventually yield regret that scaled like $\tilde
O(r^{3/2} \sqrt{T})$ as opposed to the $\tilde O(r \sqrt{T})$ accomplished here; further, the
present proof is short.
We also note that there is a norm mis-match in the lemma above since we ideally want to measure $\bar x_t$ in the $\|\cdot\|_{\bar \Omega_t^{-1}}$ norm. To relate these two norms we note that the following is true on the event that $\bar Z_t$ is well approximated:
\begin{lemma}
\label{le:norm_mismatch}
On $C_t^{\rm latent}$, we have $\|x\|_{\Omega_t^{-1}} \leq (1 + \alpha) \|x\|_{\bar \Omega_t^{-1}}$ for all $x \in \mathbb{R}^{r+1}$.
\end{lemma}
The proof of this Lemma is provided in Appendix \ref{sec:proof-norm-mismatch}. The proof crucially uses the fact that the actions selected under distinct rotations of $\hat Z_t$ are in fact equal in distribution so that we can assume a canonical rotation without loss.
Finally, we observe that the probability that the optimal action is selected is lower bounded by a constant:
\begin{lemma}
\label{le:opt_action_lb}
On $C_t$, the optimal action is selected with at least constant probability,
\[
\mathbb{P}\left(a_{t+1} = a^* \mid C_{t} \right) \geq \frac{1}{4}.
\]
\end{lemma}
So equipped, we have
\[
\begin{split}
  R_{t+1}
  &\leq
  2\beta_t \left(
  \|\bar{x}_{t+1}^*\|_{\Omega_t^{-1}}
  +
   \|\bar{x}_{t+1}\|_{\Omega_t ^{-1}}
    \right)
   \\
   &\leq
   2\beta_t (1 + \alpha) \left(
   \| \bar{x}_{t+1}^* \|_{\bar{\Omega}_t^{-1}} + \| \bar{x}_{t+1} \|_{\bar{\Omega}_t ^{-1}}
   \right)
    \\
   &\leq 2\beta_t (1 + \alpha ) \left(\frac{\mathbb{E}[ \|\bar{x}_{t+1}\|_{\bar{\Omega}_t^{-1}} \mid C_t]}{\mathbb{P}\left(\bar{x}_{t+1} = \bar{x}_{t+1}^* \mid C_{t}\right)} + \|\bar{x}_{t+1}\|_{\bar{\Omega}_t ^{-1}} \right)
   \\
   &\leq 2\beta_t (1 + \alpha ) \left(4 \mathbb{E}[ \|\bar{x}_{t+1}\|_{\bar{\Omega}_t^{-1}} \mid C_t] + \|\bar{x}_{t+1}\|_{\bar{\Omega}_t ^{-1}} \right)
\end{split}
\]
where the first inequality is Lemma~\ref{le:single_step_ub1}, the second inequality is via Lemma~\ref{le:norm_mismatch}, the third is simply by the law of total expectation, and the final inequality is via Lemma~\ref{le:opt_action_lb}. Taking expectations conditioned on $C_t$ now yields the result of the proposition. In the remainder of this Section, we prove Lemmas~\ref{le:single_step_ub1} and \ref{le:opt_action_lb}.
\newline
\newline
\noindent \textbf{\textsf Proof of Lemma~\ref{le:single_step_ub1}:} First observe that if $0 \notin [\hat \tau_{t} - \beta_t \hat \sigma_t, \hat \tau_{t} + \beta_t \hat \sigma_t]$, then $a_{t+1} = a^*$. On the other hand, if $0 \in [\hat \tau_{t} - \beta_t \hat \sigma_t, \hat \tau_{t} + \beta_t \hat \sigma_t]$, then $|\tau^*| \leq 2 \beta_t \hat \sigma_t$. Consequently, if $a_{t+1} \neq a^*$, then $R_{t+1} = |\tau^*| \leq 2 \beta_t \hat \sigma_t$. Further,
$$
\hat \sigma_t = \|e_1\|_{\Omega_t^{-1}} = \| x^*_{t+1} - \bar x_{t+1}\|_{\Omega_t^{-1}} \leq \| x^*_{t+1} \|_{\Omega_t^{-1}} + \| \bar x_{t+1}\|_{\Omega_t^{-1}},
$$
completing the proof.
\newline
\newline
\noindent \textbf{\textsf Proof of Lemma~\ref{le:opt_action_lb}:} Suppose $\tau^{*} \geq 0.$ Then $a_{t+1}=a^{*}$ whenever $\tilde{\tau}_{t} \geq 0$. Note that
\begin{align*}
\mathbb{P}(\tilde{\tau}_{t} \geq 0) \geq \mathbb{P}(\tilde{\tau}_{t} \geq \tau^{*}) \geq \mathbb{P}(\tilde{\tau}_t \geq \hat{\tau}_{t}+\beta_{t} \hat{\sigma}_t / 2) = \frac{1}{4}.
\end{align*}
When $\tau^{*} < 0$, the bound holds by symmetry.

\newcommand{\tauSC}{\tau^{\mathrm{SC}}}
\newcommand{\tauTS}{\tau^{\mathrm{SCTS}}}
\providecommand{\norm}[1]{\left\lVert\mspace{1mu} #1 \mspace{1mu}\right\rVert}
\providecommand{\R}{\mathbb{R}}

\section{Inference}\label{sec:inference}

Having run SCTS up to time $T$, we must produce an estimate of the treatment effect, $\tau^{\rm SCTS}$.  To this end, we propose the use of one of two estimators. The first is simply to set $\tau^{\rm SCTS} = \hat \tau_T$ if the intervention was used over at least $T/2$ epochs and to set $\tau^{\rm SCTS} = 0$ otherwise. A distinct alternative considered in this section is to compute the usual synthetic control estimator~\cite{abadie2010synthetic} and to set $\tau^{\rm SCTS}$ to this value on the event that the intervention was used over at least $T/2$ epochs (and to $0$ otherwise). This section will show that such an estimator
\begin{itemize}
\item enjoys identical confidence intervals to the vanilla SC estimator on the event that the treatment effect is non-negative;
\item is with high probability $0$ when the treatment effect is negative.
\end{itemize}
As such we make precise the promise set out earlier in the paper of allowing for precise estimates of the treatment effect when they matter (i.e., when the treatment effect is non-negative, so as to permit a cost-benefit analysis of implementation, say), while continuing to conclude that the treatment is ineffective when it is not.%
\footnote{It may be feasible to construct estimators for which we can provide limiting distributions with additional distributional assumptions on the setup such as in  \cite{li2020statistical,deshpandeAccurateInferenceAdaptive2018} or to more complicated synthetic control estimators \cite{doudchenko2016balancing,li2017estimation,arkhangelsky2019synthetic}. We leave this for future work.}

\subsection{Vanilla Synthetic Control}
Before describing  our estimator we make precise what can be accomplished with synthetic control and a fixed design \cite{abadie2010synthetic}. In particular, SC seeks to find a linear combination of donor units to match the experimental unit based on the observations from the pre-treatment period. SC assumes the existence of weights $w_1, w_2, \dotsc, w_{n}$ such that
$
y^{0}_t = \sum_{i=1}^{n} w_i y^{i}_{t}
$
for all times $t$ in the pre-treatment period $\{-T_0+1, \dots, 0\}$. These weights are required to be non-negative, and must sum to one. This requirement that the synthetic control be constructed as a {\em convex} (as opposed to affine) combination of donor units serves effectively as a regularization mechanism.
SC estimates the treatment effect $\tauSC$ by averaging over the differences between the synthetic control so constructed and the observed $y^{0}_t$ over the treatment period, i.e.,
\begin{align*}
\tau^{\text{SC}} := \frac{1}{T} \sum_{t=1}^{T}  \left(y^{0}_t - \sum_{i=1}^{n} w_i y^{i}_t\right).
\end{align*}

It is difficult to calculate a limiting distribution for $\tau^{\rm SC}$ absent further distributional assumptions. That said, the analysis of \cite{abadie2010synthetic} allows for the following high-probability confidence intervals as a corollary (see Appendix B in \cite{abadie2010synthetic}):
\begin{proposition}\label{prop:synthetic-control-guarantee}
With probability $1-O(\delta)$,
\begin{align*}
|\tauSC - \tau^{*}| \lesssim \frac{\sigma}{\sqrt{T}}\sqrt{\log(1/\delta)} + \frac{c_2 \sigma}{\sqrt{c_1 T_0}} \sqrt{\log(1/\delta)+\log(n)}.
\end{align*}
\end{proposition}
The constant
$c_1 = \sigma_{r} ({1}/{T_0}\sum_{t=-T_0+1, \dotsc, 0} \bar{z}_t \bar{z}_t^{\top})$
is a measure of how well conditioned-the subspace spanned by the common factors over the pre-treatment period is; and $c_2 \triangleq \max_{t} \|z_{t}\|$.
In the typical case with $c_1=c_2=O(1)$ and $\delta=1/\mathrm{poly}(n)$, we have
$$
|\tauSC - \tau^{*}| = \tilde{O}\left(1/\sqrt{T_0}+1/\sqrt{T}\right),
$$
which is optimal up to logarithmic terms.

\subsection{Using the Synthetic Control Estimator with an SCTS Design}

We now describe one potential estimator for the setting where the experiment design is determined
by SCTS. Let $w_1, w_2, \dotsc, w_{n}$ be the same weights used by the vanilla SC estimator; note
that these are computed using data available exclusively over the pre-treatment period. Now let $M
\triangleq \{t \geq 1 : a_t =1\}$ be the epochs over which the intervention is applied under our dynamic SCTS design, and denote by $\tilde \tau^{\rm SC}$, the average difference between observed outcomes and the synthetic control over those epochs,
\[
\tilde \tau^{\rm SC} \triangleq \frac{1}{|M|} \sum_{t\in M} \left(y^{0}_t - \sum_{i=1}^{n} w_i y^{i}_t\right).
\]
We then propose the following estimator, $\tau^{\rm SCTS}$ for the treatment effect,
\begin{equation*}
\tau^{\rm SCTS}
\triangleq
\begin{cases}
\tilde \tau^{\rm SC} &\text{if $|M| \geq T/2$,}\\
0 &\text{if $|M| < T/2$.}
\end{cases}
\end{equation*}
This estimator then enjoys high-probability confidence intervals analogous to the vanilla synthetic control setting:
\begin{proposition}\label{prop:SCTS-full-inference-result}
Let $\tau^*$ be fixed. Then,
\begin{itemize}
\item[(a)]  When $\tau^{*} \geq 0$, with probability $1-O(\delta) - \tilde{O}(1/\sqrt{T})$,
\begin{align*}
|\tauTS - \tau^{*}| \lesssim \frac{\sigma}{\sqrt{T}}\sqrt{\log(1/\delta)} + \frac{c_2\sigma}{c_1 \sqrt{T_0}} \sqrt{\log(1/\delta) + \log(n)}.
\end{align*}\item[(c)] When $\tau^{*} < 0$, with probability at least $1 - \tilde{O}(1/\sqrt{T})$,
$$
\tauTS = 0.
$$
\end{itemize}
\end{proposition}
The appendix proves a stronger version of this result that makes precise the dependence of the rate on $\tau^*$, and allows for meaningful confidence intervals provided $|\tau^*| = \omega(1/\sqrt{T})$. As outlined at the outset, the result above shows high probability confidence intervals analogous to the vanilla synthetic control setting on the event that the treatment effect is non-negative. On the event where the treatment effect is negative, we see that we only learn that this is the case but do not recover a precise estimate of the effect. As argued earlier, in the practical settings we care about, a precise estimate of the treatment effect is typically not as important when the effect is negative. As a result the regret gains made possible via the use of SCTS likely constitute a beneficial tradeoff relative to the inference possible under $\tauTS$.

As discussed earlier, the high-probability confidence intervals described in this section are typically quite conservative in practice. As such, in our experiments, we will explore the use of hypothesis tests based on a certain re-randomization of the data and confidence intervals derived from inverting these tests.

\section{Experiments}

This section undertakes an experimental evaluation of SCTS using both synthetic and real-world datasets. In the latter datasets, it is unclear that the synthetic control model holds (i.e., it is unclear that the observed data can be explained by a low rank factor model). We compare SCTS against both the standard fixed design (where $a_t=1$ over the entire treatment period), as well as a switchback design (where $a_t$ is set to $1$ with probability $1/2$ independently for each epoch in the treatment period) \cite{brandt1938tests}. In the case of these incumbent designs, we estimate the treatment effect using state-of-the-art estimators gleaned from recent advances in `robust' synthetic control and panel data regression. Our experiments will illustrate the following salient features of SCTS:

\begin{enumerate}
  \item The fraction of time SCTS chooses a sub-optimal action is small. In contrast, by definition, the switchback design picks a sub-optimal action half of the time, while the fixed design picks the sub-optimal action all of the time when the treatment effect is negative. Despite the material reduction in regret, our estimator of the treatment effect under SCTS achieves relative error comparable to the competing designs when the treatment effect is positive. This same estimator correctly identified that the treatment effect was negative in all experiments where this was the case.
  \item In the case of real world data where a low-rank factor model provides at best a `rough' fit to the observations, we observe the same relative merits alluded to above. In that case, however, we observe an additional merit for SCTS, where the relative error in estimating the treatment effect is actually substantially lower than that for the fixed design (and comparable to that of the switchback design). The reason is that in the SCTS setting, our estimator can take advantage of data collected over the treatment period in estimating factor loadings and this appears to be particularly valuable when the common factor process is non-stationary.
  \item Re-randomization tests \cite{fisher1966design} provide a means to construct well-powered hypothesis tests and confidence intervals for SCTS, despite the inferential challenges introduced by adaptive treatment assignment.
\end{enumerate}



\subsection{Low Regret and Estimation Error on Synthetic Data}

Our first set of experiments seeks to establish that SCTS incurs low regret while recovering the treatment effect accurately. We first consider this on a synthetic set of problems that we describe next.
\newline
\newline
\noindent {\em Experimental setup:} We experiment with a synthetically generated dataset. We generate the latent factors and loadings $\Lambda, \bar{Z}_{T}$ with entries distributed i.i.d. as $\mathcal{N}(0, 1)$, with $n=1000, T_0=500, T=500$ and $r=50$. Similarly, we generate $\lambda^*$ with i.i.d. $\mathcal{N}(0, 1)$ entries. We experiment with a signal-to-noise ratio of $1$; i.e., $|\tau^*| / \sigma = 1$, and vary the sign of $\tau^*$.
\newline
\newline
\noindent {\em Estimation:} There are a variety of treatment effect estimators we could use for any given experimental design. For SC, we report the best performance over several estimators from the literature; the best performing estimator on our synthetic examples is the robust synthetic control estimator proposed in \cite{amjad2018robust}. For the switchback design, we use our ridge regression estimator $\hat{\tau}_{T}$. For SCTS, we set $\tau^{\rm SCTS} = \hat{\tau}_{T}$ when $M \geq \frac{N}{2}$, and 0 otherwise.
\newline
\newline
\noindent {\em Results:} Table~\ref{tbl:synthetic} shows mean regret and estimation error for each algorithm averaged over $50$ problem instances.
The results comport favorably with the salient features we outlined for SCTS at the outset. Specifically:
\begin{enumerate}
\item SCTS picks a sub-optimal action over at most $2\%$ of the available testing epochs, thereby mitigating any cost of experimentation. By construction this number is $50 \%$ for the switchback design and $100\%$ for the fixed design when the treatment effective is negative.
\item Despite the above gain, we continue to recover the treatment effect accurately with the SCTS design. Specifically, the relative RMSE is comparable to the switchback and fixed design cases when the treatment effect is positive. The SCTS estimator correctly identified that the treatment effect was negative in all instances where this was the case.
\item From the plots in Figure~\ref{fig:synthetic}, we see that the relative merits of SCTS alluded to above are robust to the choice of experimentation horizon $T$.
\end{enumerate}


\begin{table}[htbp]
  \centering
\begin{tabular}{|c|c|rrr|}
\hline
 & &  SCTS & Switchback  & SC  \\
\hline
 $\tau^* < 0$ & Regret &  0.02 & 0.50 & 1.00 \\
 $\tau^* > 0$ & Regret & 0.01 & 0.50 & 0.00 \\
 $\tau^* > 0$ & RMSE & 0.06 & 0.08 & 0.06 \\
  \hline
\end{tabular}
\caption{Regret and relative RMSE, averaged over 50 random synthetic instances. Regret is normalized between 0 and 1. RMSE is normalized by $|\tau^*|$. SCTS virtually eliminates the cost of experimentation, while providing estimates of $\tau^*$ of the same quality as more costly alternatives when $\tau^* > 0$.}
  \label{tbl:synthetic}
\end{table}

\begin{figure}[htbp]
  \centering
  \includegraphics[width=\linewidth]{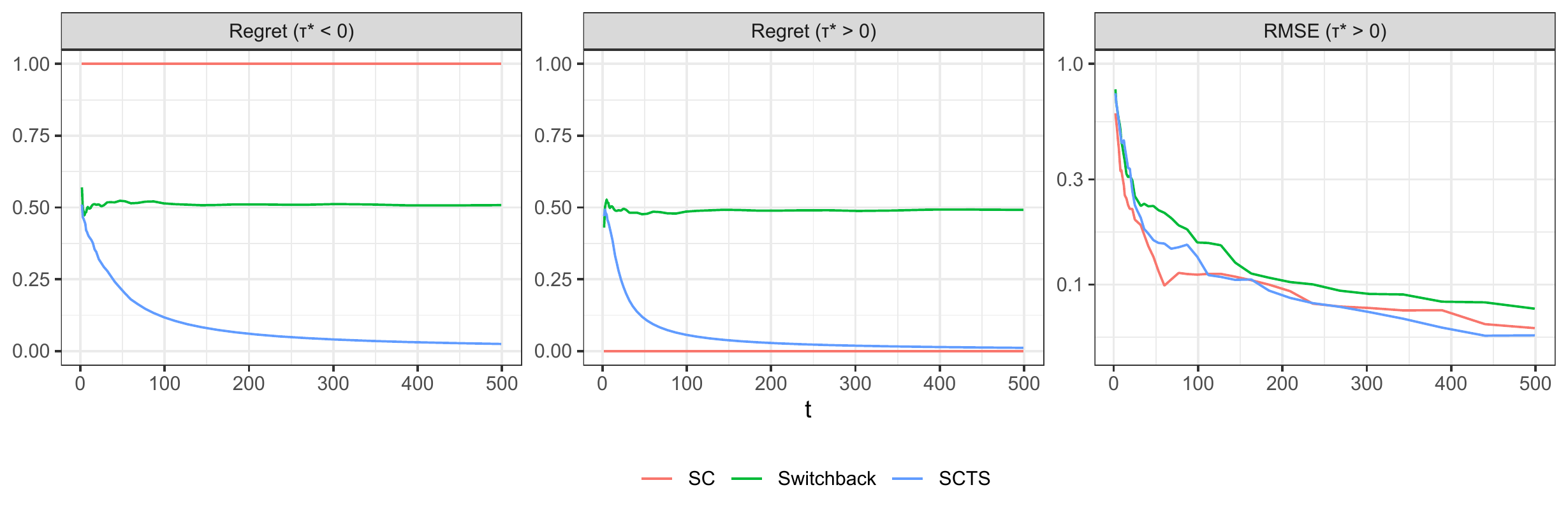}
  \caption{Regret (normalized by $t |\tau^*|$) and RMSE (normalized by $|\tau^*|$) over time, for the synthetic dataset. Unlike SC and switchback, SCTS exhibits regret vanishing over time in addition to a small RMSE. These qualities are robust to the experimentation horizon $T$.}
  \label{fig:synthetic}
\end{figure}

\subsection{Real-world data}
Our second set of experiments serves the same purpose as the earlier set, except that this time we consider real world data. As such, there is no true low-rank factor model describing the data; at best we may hope that such a model provides a good approximation to the observed data. As before, our goal will be to measure regret for SCTS as well as estimation error.
\newline
\newline
\noindent {\em Experimental Setup: } We adapt the Rossman Store Sales dataset\footnote{https://www.kaggle.com/c/rossmann-store-sales/}, which contains daily sales data for $n = 1115$ drug stores over $942$ days.  We take $T_0 = T = 471$. Letting $O \in \mathbb{R}^{n \times (T + T_0)}$ be the matrix of observations in the dataset, we generate an ensemble of 50 instances as follows. For each instance, we select a random store $j$ to be the experimental unit, with outcomes $y_{t}^{0} = O_{j, t} + \tau^* a_t$. The remaining stores constitute the control units, with observations $y_{t}^{i} = O_{i, t}$. Viewing the rank $r$ now as an algorithmic hyper-parameter, we use $r=70$ in our experiments. This choice was made via cross-validation on the pre-treatment period, as in \cite{owenBiCrossValidationSVDNonnegative2009}.
As before, we experiment with two values of $\tau^*$: $\tau^* = \sigma$ and $\tau^* = -\sigma $, where $\sigma^2$ is estimated as mean squared error of $O$ relative to its best rank $70$ approximation.
\newline
\newline
\noindent {\em Estimation: }We use the same set of estimators here as in the previous set of experiments with synthetic data.
%
%
\newline
\newline
\noindent {\em Results: }At the outset, we note that the model equations \eqref{eq:reward-model}--\eqref{eq:donor-model} on which any of our designs or estimation approaches are predicated do not hold exactly in this setup. In particular, approximation error essentially precludes the `noise' in our rank $70$ model from being Gaussian or i.i.d. Referring to Table~\ref{tbl:rossman}, we observe:
\begin{enumerate}
\item SCTS picks a sub-optimal action over at most $13\%$ of the available testing epochs, mitigating the cost of experimentation.
\item Despite the above gain, we continue to recover the treatment effect accurately with the SCTS design. Specifically, the relative RMSE is comparable to the switchback design. The SCTS estimator correctly identified that the treatment effect was negative in all instances where this was the case.
\item Especially interesting is that relative RMSE is substantially lower than that for the fixed design. We attribute this to the fact that the SCTS estimator (and the switchback estimator) are effectively able to utilize data obtained over the test period to continually refine the factor loadings defining the synthetic control -- this appears to be particularly valuable in settings such as this dataset where the common factor process is not stationary.
  \item  As an aside, for the switchback design we also compute the naive difference-in-means estimator\footnote{This simply takes the difference between the average reward for periods when $a_t = 0$, and the average reward for periods when $a_t = 1$.}, which does not make use of the control observations. Doing so results in a much higher relative RMSE, illustrating the value of a synthetic control in estimating the treatment effect with that design.
\item Finally, from the plots in Figure~\ref{fig:rossman}, we see that the relative merits of SCTS alluded to above are robust to the choice of experimentation horizon $T$.
\end{enumerate}

%

\begin{table}[htbp]
  \centering
\begin{tabular}{|c|c|rrr|}
\hline
 & &  SCTS  & Switchback  & SC  \\
\hline
$\tau^* < 0$ &  Regret &  0.13 & 0.50 & 1.00 \\
$\tau^* > 0$ & Regret  &  0.09 & 0.50 & 0.00 \\
$\tau^* > 0$ & RMSE  &  0.12 & 0.07,0.91* & 0.57 \\
  \hline
\end{tabular}
\caption{Regret and relative estimation error, averaged over random instances generated from the Rossman dataset. Regret is normalized between 0 and 1. RMSE is normalized by $|\tau^*|$. In addition to low regret, SCTS even produces better quality estimates of $\tau^*$ in this setting, compared to SC. For Switchback, we report RMSE for two estimators: $\hat{\tau}_T$ (RMSE=0.07), and a simple difference in means with no synthetic controls (RMSE=0.91), highlighting the importance of synthetic controls even with switchback designs.}
  \label{tbl:rossman}
\end{table}

\begin{figure}[htbp]
  \centering
  \includegraphics[width=\linewidth]{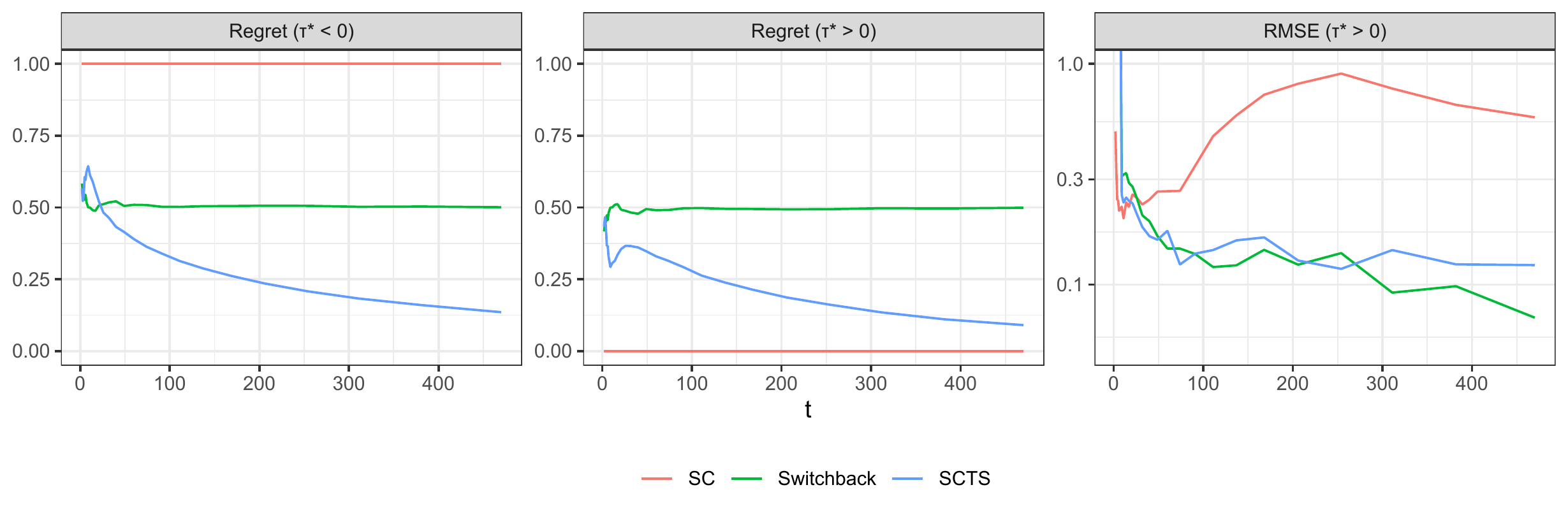}
  \caption{Regret (normalized by $t |\tau^*|$) and RMSE (normalized by $|\tau^*|$) over time, averaged over instances from the Rossman dataset. Unlike SC and Switchback, SCTS exhibits regret vanishing over time in addition to a small RMSE. In particular SCTS and Switchback estimators both display much lower RMSE than SC. These qualities hold essentially for all $t$ in the horizon.}
  \label{fig:rossman}
\end{figure}

\subsection{A Non-Parametric Approach to Inference}
Finally, we explore a re-randomization approach to inference for SCTS. In particular, while the high-probability concentration bounds of Lemma~\ref{thm:clean-prob} can be used to provide confidence intervals for $\hat{\tau}$, they assume that the structural model \eqref{eq:reward-model} is realizable, and tend to be conservative in practice. On the other hand, post-bandit inference techniques such as those in \cite{deshpandeAccurateInferenceAdaptive2018,bibautPostContextualBanditInference2021} could possibly be adapted to our setting, but do not apply immediately (see Section~\ref{sec:literature} for discussion). This latter direction remains an exciting direction for future work.

Here we propose a re-randomization test similar to that of \cite{bojinovDesignAnalysisSwitchback2020} to construct a hypothesis test for a sharp null that the treatment effect is some specific value. We then obtain confidence intervals by inverting this hypothesis test, as described in \cite{imbens2015causal}. The overall conclusion in this section is that (a) our hypothesis test is highly powered for relatively low values of SNR where the treatment effect is dominated by the noise and (b) our confidence intervals attain nearly ideal coverage even for very low SNR. All of these experiments are run on the real data setup described in the preceding section.
\newline
\newline
\noindent {\em A Re-Randomized Hypothesis Test and Confidence Intervals: }We test the sharp null hypothesis $H_\tau$ that the treatment effect is some constant $\tau$, for all $t$. To implement such a test, suppose that we have run SCTS for $T$ time steps, obtaining an estimator $\hat{\tau}_{T}$. We take this estimator to be our test statistic. We can then construct an approximate hypothesis test, at significance level $\alpha$, as follows:

\begin{enumerate}
\item We are given an observed trajectory of interventions under SCTS, $\{a^{\rm hist}_t\}$ and the corresponding observations on the experimental unit $\{y^{0, {\rm hist}}_t\}$.
\item We next draw $k$ samples $\tau^{(1)} \ldots \tau^{(k)}$ of the test statistic under the null hypothesis. We do so by re-running SCTS, but assuming that we observe the sequence of outcomes $y_{t}^{0,{\rm hist}} + \tau a_t - \tau a^{\rm hist}_t$
\item  We can then approximate the p-value of the test statistic as one minus the proportion of the samples $\tau^{(1)} \ldots \tau ^{(k)}$ which are less than $\hat{\tau}_{T}$.
\item  We reject $H_\tau$ if the p-value is less than the significance level $\alpha$.
\end{enumerate}
We may now construct confidence intervals by `inverting' the above re-randomized hypothesis test, as described in~\cite{imbens2015causal}.
Precisely, for every null $H_{\tau}$ for some $\tau \in \mathbb{R}$, we can implement a re-randomization test and decide whether to reject $H_{\tau}$. The confidence set is then the set of $\tau$ values for which we do not reject the corresponding null $H_{\tau}$.
\newline
\noindent {\em Results: } We assess our re-randomization test on 100 problem instances generated from the Rossman sales dataset, as above. We draw $k = 100$ samples of the test statistic for each instance and choose the the significance level to be $\alpha = 0.1$.
The results in Table~\ref{tbl:permutation-power} show that this test is highly powered even when the treatment effect is dominated by the noise (i.e., at an SNR of $0.2$). Power is already nearly $1$ at an SNR of $1$. Further, we see that the coverage of the test is close to ideal (given the significance level of $0.1$, ideal here is $0.9$) over a broad range of SNRs from $0.01$ to $1$. In summary, we conclude that the re-randomization tests and corresponding confidence intervals reported here are adequate for inference even when SNR is low.


\begin{table}
\begin{tabular}{|r|rrrrr|}
\hline
 &  $\mathrm{SNR} = 0.01$ & 0.02 & 0.1 & 0.2 & 1\\
  \hline
Coverage & 0.89 & 0.87 & 0.87 & 0.90 & 0.89 \\
Power    & 0.14 & 0.17 & 0.51 & 0.81 & 0.98 \\
\hline
\end{tabular}
  \caption{Performance of the re-randomization test and confidence intervals on Rossman problem instances, as a function of the effect size (normalized as SNR). As expected, coverage attains roughly the nominal level $1 - \alpha$ where $\alpha = 0.1$, while power increases quickly as we increase the effect size.}
  \label{tbl:permutation-power}
\end{table}

\section{Limitations and Open Directions}

We motivated our dynamic design by the real-world setting where, in order to deal with issues of interference, we must select treatment units that are `coarse'. We discussed at length that these coarse units often necessitate synthetic controls to enable inference. This is true even in switchback designs where using the `untreated' history of the experimental unit is not a reliable control, especially when the data has non-stationary components. At the same time, by virtue of being coarse, the cost of experimentation as embodied by the number of epochs over which a potentially sub-optimal action was taken, is no longer trivially ignored. 

The SCTS approach is an attempt to provide a new dynamic design that addresses these issues. In practical settings, however, one must contend with real-world issues not addressed by the model studied here. We outline these issues here and present directions for future work that might serve to address them: 
\newline
\newline
\noindent{\em Spill-over Effects: }In dynamic designs --- the switchback design is a simple example, SCTS is another --- one often cares about `spill-over' effects wherein a treatment applied in one epoch might influence outcomes in subsequent epochs.  The fix to this issue is to typically allow a `burn-in' period that ignores epochs impacted by such spill-overs. These burn-in periods typically precede and follow a switch from one type of treatment to another. Whereas we have not posited a formal model, it is reasonable to conjecture that in our setting, one could employ a similar strategy. Since the number of switches in SCTS is small (i.e. $O(\sqrt{T})$), the added regret from such burn-in periods will scale sub-linearly with the horizon.
\newline
\newline
\noindent{\em Interventions over Consecutive Epochs: }For some interventions, it may be practically necessary (for instance, from a consumer experience standpoint) that any intervention be maintained over a certain minimum number of consecutive epochs. We believe this to be an important area for future work, closely related to notions of switching costs and batching in the bandit literature. A number of flavors of this problem have been considered in recent years, including incorporating switching costs \cite{dekelBanditsSwitchingCosts2014}, and batching that makes the decision to stop using a potential intervention irrevocable \cite{perchetBatchedBanditProblems2016}. Very recently, \cite{esfandiariRegretBoundsBatched2021,hanSequentialBatchLearning2020a} have extended the batched bandit formalism to linear contextual bandits. Whereas none of these models precisely address the modeling need above, they provide a very reasonable foundation for a potential extension to SCTS that incorporates the constraint that any intervention must be pursued for a certain minimal number of consecutive epochs.
\newline
\newline
\noindent{\em Post-Bandit Inference: }This is an issue we have discussed earlier. Specifically, while we can establish high-probability confidence intervals (that are conservative) and re-randomization tests (that appear to work well practically), we would ideally like to construct estimators with limiting distributions that permit powerful inference. The growing post-bandit inference literature \cite{bibautPostContextualBanditInference2021,hadadConfidenceIntervalsPolicy2021,deshpandeAccurateInferenceAdaptive2018} provides an approach to accomplishing this goal. The primary road block here is that existing proposals ask for a lower bound on the rate of decay of exploration, and it is not clear that such a lower bound is met by the current proposal.

\bibliographystyle{apalike}
\bibliography{bib}

\newpage
\appendix
\section{Failure of UCB}\label{sec:fail-UCB}


Here, we construct a class of problem instances to show that the UCB algorithm incurs linear regret.

To begin, consider the case when $r=1$ and $\sigma=0$ (i.e. the noiseless scenario). For any $\tau^{*} < 0$, let $\lambda^{0} = 1 - \tau^{*}$ and $\lambda^{i} = 1$ for $i \in [n].$ Further, let $\bar{z}_{t}=1.$ Then the observation is given by
\begin{align*}
y_{t}^{i} = 1, i \in [n]
\end{align*}
and 
\begin{align*}
y_{t}^{0} = 1 - \tau^{*} + a_t \tau^{*}.
\end{align*}

Recall that UCB chooses actions by the following procedure: (i) compute $\hat{z}_t$ by estimating the common factors through SVD; (ii) solve the ridge regression problem to obtain $\hat{\tau}_{t}$ and its `variance' estimate $\hat{\sigma}_{t}$; (iii) play $a_t=1$ if $\hat{\tau}_t + \beta_{t} \hat{\sigma}_{t} \geq 0$, and $a_t= 0$ otherwise. Next, we will show that this algorithm will constantly choose $a_t=1$, for any ridge regularizer $\rho$ and any sequence of $\beta_{t}$, thereby incurring $O(T|\tau^{*}|)$ regret.
\newline
\textbf{Estimating shared common factors.} By the SVD of $Y_{t}$, one has
\begin{align*}
Y_{t} = \sqrt{nt} \hat{U}^{\top} \hat{V}
\end{align*}
with $\hat{U}_{i} = \sqrt{1/n}$ and $\hat{V}_{s} = \sqrt{1/t}.$ Then the estimator $\hat{z}_{t} \triangleq \sqrt{1/n}\sqrt{nt}\hat{V}_{s} = 1$, i.e., in the noiseless setting, $\hat{z}_{t} = \bar{z}_t = 1.$
\newline
\textbf{Ridge regression.}
Recall that we will solve the following (regularized) least squares problem at time step $t+1$.
\begin{align*}
       \underset{\tau \in \mathbb{R}, \lambda \in \mathbb{R}^{r}}{\min}
       \sum_{s \leq t} 
      \left(y^0_s - \tau a_s - \langle \lambda, \hat z_s \rangle \right)^2 
      + \rho (\tau^2 + \|\lambda\|_{2}^2)
\end{align*}
Under the constructed setting and the assumption that $a_{s}=1$ for $s\leq t$, the problem is equivalent to 
\begin{align*}
       \underset{\tau \in \mathbb{R}, \lambda \in \mathbb{R}^{r}}{\min}
       \sum_{s \leq t} 
      \left(1 - \tau - \lambda \right)^2 
      + \rho (\tau^2 + \lambda^2)
\end{align*}
which has the closed-form solution
\begin{align*}
\hat{\tau}_{t} =  \hat{\lambda}_{t} = \frac{t}{\rho + 2t}.
\end{align*}
\textbf{Action decision.} Since $\hat{\tau}_{t} > 0 $, we have $\hat{\tau}_{t} + \beta_{t} \hat{\sigma}_t \geq 0$ and hence $a_{t+1}  = 1$. This implies the UCB algorithm will play $a_t = 1$ for all $t$, thereby incurring regret $T|\tau^{*}|$.

\section{Canonical Decomposition}\label{sec:canonical-decomp}
Consider the structural model given by parameters $\lambda^{*\prime}, \Lambda', \bar{Z}_{T}^{\prime}$

\begin{align*}
    \mathbb{E} [y_t^{0}] = \left\langle {\lambda^{*}}^{\prime} , \bar{z}^{\prime}_t \right\rangle + \tau^* a_t &&
    \mathbb{E} [Y_{T}] = \Lambda' (\bar{Z}_T')^\top
  \end{align*}
It is easy to see that the selections of $\lambda^{*}, \bar{Z}_{T}', \Lambda'$ are not unique due to the free rotations. 

Let $\mathbb{E} [Y_{T}] = \bar{U}_T \bar{\Sigma}_{T}\bar{V}^\top_{T}$ be an SVD of $\mathbb{E} [Y_{T}]$ and let $\bar{Z}_T = \frac{1}{\sqrt{n}} \bar{V}_T \bar{\Sigma}_T$. We aim to show that $\bar{Z}_{T}$ can constitute a canonical representation. In particular, it is sufficient to show that there exist unique $\lambda^{*}$ and $\Lambda$ such that
\begin{align*}
    \mathbb{E} [y_t^{0}] = \left\langle \lambda^{*} , \bar{z}_t \right\rangle + \tau^* a_t &&
    \mathbb{E} [Y_{T}] = \Lambda (\bar{Z}_T)^\top
  \end{align*}

Since $\mathbb{E}\left[Y^\top_{T}\right]$ is exactly rank $r$, the column spaces of $\bar{Z}_T$ and $\bar{Z}_T^{\prime}$ must be the same. Therefore there must exist an invertible matrix $H$ such that $\bar{Z}_{T} = \bar{Z}_{T}^{\prime}H$, which induces unique choices of $\lambda^{*}$ and $\Lambda$:
\begin{align*}
\lambda^{*} =  H{\lambda^{*}}^{\prime} \\
\Lambda = \Lambda'H.
\end{align*}

%

\section{Full Version of Theorem~\ref{thm:regret}}
\label{sec:theorem-statement}

Here we state the full version of our main theorem, Theorem~\ref{thm:regret}, including explicit dependencies on all problem parameters.

\noindent \begin{theorem}
  Under Assumption~\ref{ass:bar_y_decomp}, the regret of SCTS is
  \begin{equation*}
R(T) \leq 4 |\tau^*| + 10 \beta_t (1 + \alpha) B \sqrt{(r + 1) T \log \left(1 + \frac{T B^2}{r+1}\right) }
  \end{equation*}
  where we define
  \begin{align*}
    \alpha   &= 20 \sigma \sqrt{\frac{n \lor T}{n}} \\
    \beta_t &= 2 \sigma \sqrt{2(r+1) \log t\left(r+1 + t (B + 1 + \alpha)\right)} + (\|\lambda^{*}\|+|\tau^{*}|)(1 + \alpha)
  \end{align*}
\end{theorem}

\section{Proof of Proposition \ref{prop:clean-event}}\label{sec:proof-clean-event}
We proceed with the proof by showing that $\mathbb{P} (C_{t}^{\rm latent})$ and $\mathbb{P} (C_{t}^{\rm est})$ are well controlled, separately.
\subsection {Controlling $\mathbb{P} (C_{t}^{\rm latent})$} Recall the definition of $C_t^{\rm latent}$:
\[
C_t^{\rm latent} = \left\{ 
\inf_{\Phi \in \mathcal{O}_r}
\|\bar{Z}_{t} - \hat{Z}_{t} \Phi \| \leq \alpha
\right\}.
\] 
The following lemma, based on a generalization of the Davis-Kahan bound, provides the desired bounds for $\mathbb{P} (C_{t}^{\rm latent}).$
\begin{lemma}
Let $\alpha \triangleq 20\sigma \sqrt{\frac{n \lor T}{n}}.$ Then, for all $t$, with probability $1-O(1/t^{8}),$
\begin{align}
\inf_{\Phi \in \mathcal{O}_r}
\|\bar{Z}_{t} - \hat{Z}_{t} \Phi \| \leq \alpha  \label{eq:align}
\end{align}
\end{lemma}

\begin{proof}

We are interested in three SVDs at time $t$:

        \begin{itemize}
          \item $\bar{Y}_{T} = \bar{U}_{T} \bar{S}_{T} \bar{V}^\top_{T} = \sqrt{n}\bar{U}_{T}\bar{Z}_{T}$, where $\bar{Z}_{T}$ is the ``canonical'' representation of the latent covariates.
          \item $\hat{U}_t \hat{S}_t \hat{V}_t^\top$, which is the SVD of $Y_t$ truncated to $r$ singular values.
          \item $\bar{Y}_{t} = \bar{U}_t \bar{S}_t \bar{V}^\top_t$, i.e.  the SVD of the mean control outcomes up to time $t$.
        \end{itemize}

        We will use this last quantity $\bar{V}_t \bar{S}_t$ is to relate the other two quantities: to $\hat{V}_t \hat{S}_t$ via Theorem~\ref{thm:rotation-ZPhi}, and to $\sqrt{n} \bar{Z}_{t} \triangleq \bar{V}_{T, :t}\bar{S}$ via a rotation, where $\bar{V}_{T, :t}$ denotes the first $t$ rows of $\bar{V}_{T}.$

        To see this, observe that $\bar{Y}_t = \bar{U}_t \bar{S}_t \bar{V}^\top_t = \bar{U}_{T} \bar{S}_{T} (\bar{V}^\top_{T,:t})$, and therefore $\frac{1}{\sqrt{n}}\bar{V}_t \bar{S}_t  \bar{U}_t^\top \bar{U}_{T} = \frac{1}{\sqrt{n}} \bar{V}_{T,:t} \bar{S} = \bar{Z}_{t}$. We can further show that $\bar{U}_t^\top \bar{U}$ is unitary: for $t \geq T_{0}$ (for some constant $T_{0}$), the matrix $\bar{U}_t$ must have the same column space as the matrix $\bar{U}_{T}$, which implies that $\bar{U}_t^\top \bar{U}_{T} \bar{U}_{T}^\top \bar{U}_t = I$ and therefore $\bar{U}_t^\top \bar{U}_{T}$  is unitary.
        
     Note that our estimate of the latent covariates at time $t$ is $\hat{Z}_t = \frac{1}{\sqrt{n}}\hat{V}_t \hat{S}_t$. Then

\begin{align*}
 \inf_{\Phi \in \mathcal{O}_r}\|\hat{Z}_t\Phi - \bar{Z}_t\| &=  \frac{1}{\sqrt{n}}\inf_{\Phi \in \mathcal{O}_r}\|\hat{V}_t \hat{S}_t\Phi - \bar{V}_t\bar{S}_{t}\bar{U}_{t}^{\top}\bar{U}_{T}\|\\
 &\overset{(i)}{=} \frac{1}{\sqrt{n}}\inf_{\Phi \in \mathcal{O}_r}\|\hat{V}_t \hat{S}_t\Phi - \bar{V}_t\bar{S}_{t}\|\\
 &\leq 4\frac{\|E_{t}\|}{\sqrt{n}}
\end{align*}

where (i) is due to that $\bar{U}^{\top}_{t}\bar{U}_{T}$ is a rotation and (ii) uses \cref{thm:rotation-ZPhi}.\footnote{A direct application of the Davis-Kahan theorem would have resulted in a dependence on the condition number of $\bar{Y}_t$, $\sigma_1(\bar{Y}_t)/\sigma_r(\bar{Y}_{t})$. We conducted a more refined analysis and obtained an improved bound.} Then, using Theorem~\ref{thm:matrix-concentration}, we have $\|E_{t}\| \leq 5\sigma\sqrt{n \lor t}$ with probability $1-\frac{1}{(n\lor t)^{8}}$. This completes the proof.
\end{proof}

\subsection{Controlling $\mathbb{P} (C_{t}^{\rm est})$}
Next, we bound the probability $\mathbb{P} (C_{t}^{\rm est})$. Recall that $C_t^{\rm est} = \left\{|\tau^* - \hat{\tau}_t| \leq \beta_t \hat \sigma_t / 2 \right\}$, where we define $\beta_t$ explicitly below.

\begin{lemma}
  \label{le:clean-est-prob}
With probability $1 - \frac{2}{t^2}$, it holds that $|\tau^* - \hat{\tau}_t| \leq \beta_t \hat \sigma_t / 2 $, where
  \begin{equation*}
    \beta_t = 2 \sigma \sqrt{2(r+1) \log t\left(r+1 + t (B + 1 + 20 \sigma \sqrt{(n \lor T) / n})\right)} + (\|\lambda^{*}\|+|\tau^{*}|)(20\sigma\sqrt{(n \lor T)/n} + 1)
  \end{equation*}
\end{lemma}

Note that $\hat\theta_t^\top \triangleq [\hat \tau_t \ \hat \lambda^\top_t]$ is the solution of the following quadratic program:
\begin{align} \label{eq:ridge-optimization}
       \underset{\tau \in \mathbb{R}, \lambda \in \mathbb{R}^{r}}{\min}
       \sum_{s \leq t} 
      \left(y^0_s - \tau a_s - \langle \lambda, \hat z_s \rangle \right)^2 
      + \rho (\tau^2 + \|\lambda\|_{2}^2)
\end{align}
with $y^{0}_s = \tau^{*} a_s + \langle \lambda^{*}, \bar z_s \rangle +\epsilon_{s}^{0}.$ Denote $x_{s}^\top \triangleq [a_s \ \hat z_{s}^\top]$, let $\Omega_t = \rho I + \sum_{s \leq t} x_{s}x_{s}^\top$ be the `precision' matrix of $\hat{\theta}_{t}$. The `variance' estimator is

\begin{align*}
\hat \sigma^2_t  = (\Omega_t)^{-1}_{1,1}.
\end{align*}

We address the following two issues to complete the proof.

\begin{itemize}
\item[1.] Due to the rotation ambiguity, there is no direct connection between $\hat{z}_{s}$ and $\bar{z}_{s}.$ That is to say, we can only bound $\|\Phi\hat{z}_{s} - \bar{z}_{s}\|$ for some unknown rotation $\Phi.$ For addressing this, we show that $\hat{\tau}_t,\hat{\sigma}_{t}$ are invariant to the rotation $\Phi$, hence one can rewrite $\Phi\hat{z}_{s}$ as $\hat{z}_{s}$ without loss, for the purpose of analysis. See \cref{sec:invariance-rotation}.
\item[2.] Given the bound $\|\hat{z}_{s} - \bar{z}_{s}\|$, the problem reduces to an analysis of ridge regression with errors-in-variables, with an adapted design matrix. We adapt a typical bound from LinUCB  \cite{abbasi-yadkoriImprovedAlgorithmsLinear2011} for ridge regression without errors-in-variables to our setting. See \cref{sec:CI-ridge-regression}.
\end{itemize}

\subsubsection{Invariance to Rotation} \label{sec:invariance-rotation}
To begin, we show that under our SCTS algorithm, $\hat{\tau}_{t}, \hat{\sigma}_{t}$ are invariant to any rotation of the estimated contexts $\hat{Z}_t.$ This then enables us to assume that $\hat{Z}_{t}$ aligns with $\bar{Z}_{t}$ by the best rotation, given by \cref{eq:align}, without loss.

Specifically, suppose the observed context were $\Phi\hat{z}_{s}$ instead $\hat{z}_{s}$ for some rotation $\Phi \in \mathcal{O}^{r\times r}$, and we wanted to solve the optimization problem:
\begin{align}\label{eq:ridge-optimization-rotation}
       \underset{\tau \in \mathbb{R}, \lambda \in \mathbb{R}^{r}}{\min}
       \sum_{s \leq t} 
      \left(y^0_s - \tau a_s - \langle \lambda, \Phi\hat z_s \rangle \right)^2 
      + \rho (\tau^2 + \|\lambda\|_{2}^2)
\end{align}
Let $\breve \tau_t, \breve \lambda_t$ be the corresponding optimal solution. One can easily see that

\begin{align*}
\breve \tau_t &= \hat{\tau_t}\\
\breve \lambda_t &= \Phi \hat{\lambda}_t
\end{align*}

by the equivalence between \cref{eq:ridge-optimization} and \cref{eq:ridge-optimization-rotation} through the corresponding transformation. This  implies that $\hat{\tau}_{t}$ is invariant to the rotation $\Phi$. Further, let $\breve{x}_{s}^\top \triangleq [a_s \; (\Phi\hat{z}_{s})^\top]$ and
\begin{align*}
\breve{\Omega}_t
&= 
\rho I + \sum_{s \leq t} \breve{x}_{s} \breve{x}_{s}^\top\\
&= \rho I + \sum_{s \leq t} \begin{bmatrix} a_{s} \\ \Phi \hat{z}_{s} \end{bmatrix} \begin{bmatrix} a_{s} \\ \Phi \hat{z}_{s} \end{bmatrix}^\top \\
&= I_{1,\Phi} \hat\Omega_{t} I_{1,\Phi}^{\top}
\end{align*}
where
$$
I_{1,\Phi} \triangleq \begin{pmatrix} 1 &  \bold{0}_{1\times r}\\  \bold{0}_{r\times 1}& \Phi\end{pmatrix}.
$$
This implies that 
$$
\breve \Omega_{t}^{-1} = (I_{1,\Phi}\hat{\Omega}_{t}I_{1,\Phi}^{\top})^{-1} =  I_{1,\Phi} \hat{\Omega}_{t}^{-1} I_{1,\Phi}^{\top}.
$$
Therefore, 
$$
(\breve \Omega_{t}^{-1})_{11} = e_1^{\top} I_{1,\Phi} \hat{\Omega}_{t}^{-1} I_{1,\Phi}^{\top} e_{1} = e_1^{\top} \hat{\Omega}_{t}^{-1} e_{1} = (\hat \Omega_{t}^{-1})_{11}
$$
due to $ I_{1,\Phi}^{\top} e_{1} = e_1.$ This is to say, $\hat{\sigma}_{t}$ is also invariant to the rotation $\Phi.$

Hence, since the action chosen by SCTS only depends on $\hat{Z}_t$ via $\hat{\tau}_t$ and $\hat{\sigma}_t$, and these two quantities are invariant to the rotation of $\hat{Z}_t$, then the regret of SCTS is also invariant to the rotation of $\hat{Z}_t$. Without loss, we will rewrite $\hat{Z}_{t}\Phi$ as $\hat{Z}_{t}$ for the simplification of the analysis, if there is no ambiguity. Then under $C_{t}^{\rm latent}$, we have
\begin{align*}
\|\bar{Z}_{t} - \hat{Z}_{t}\| \leq \alpha
\end{align*}

\subsubsection{Confidence Intervals for Ridge Regression.}  \label{sec:CI-ridge-regression} First, we show a generic confidence interval for ridge regression with errors-in-variables, which adapts the bounds used in LinUCB \cite{abbasi-yadkoriImprovedAlgorithmsLinear2011}. For a matrix $A$, let $\|A\|_{2, \infty} = \max_{\|u\|_{2} = 1} \|A u\|_{\infty}$ denote the maximum norm of its rows.

\begin{proposition}
  Let $\{\mathcal{F}_t\}_{t \in \mathbb{N}}$ be a filtration. Let $\{x_t, \bar{x}_t\}_{t \in \mathbb{N}}$ be a $\mathcal{F}_{t}$ -measurable, $\mathbb{R}^{d}$ -valued stochastic process. Let $\{\varepsilon_t\}_{t \in \mathbb{N}}$ be an $\mathcal{F}_{t+1}$ -measurable, $\mathbb{R}$ -valued stochastic process. Let $\varepsilon_t$ further be $\sigma$-subgaussian conditional on $\mathcal{F}_t$. Define $\bar{X} = [\bar{x}_1\; \bar{x}_2 \ldots \bar{x}_{t}]^\top,{X} = [{x}_1\; {x}_2 \ldots {x}_{t}]^\top, \varepsilon = [\varepsilon_1 \; \varepsilon_2 \ldots \varepsilon_{t}]^\top$ and let \(y = \bar{X} \theta^* + \varepsilon\).

  Define for some $\rho > 0$ the quantities:
\begin{align*}
\Omega &= \rho I + X^\top X \\
\hat{\theta} &=  \Omega^{-1} X^\top Y \\
\beta(\delta) &= \sigma \sqrt{d \log \left(\frac{d \rho + t (\|\bar{X}\|_{2, \infty} + \|X - \bar{X}\|) }{\delta}\right) } +  \left\|\theta ^*\right\|_2 (\left\|X - \bar{X}\right\|  + \sqrt{\rho})
\end{align*}

  Then, with probability \(1 - \delta\), it holds that

\begin{equation*}
\left\|\hat{\theta} - \theta^*\right\|_{\Omega} \leq \beta(\delta)
\end{equation*}

\label{thm:eiv-ci}
\end{proposition}
\begin{proof}

\begin{enumerate}
  \item We begin by decomposing the error into the the standard usual ridge regression error, plus a term that depends on the errors in variables.
\begin{align*}
    \hat{\theta}
    &= \Omega^{-1} X^\top y \\
    &= \Omega^{-1} X^\top (\bar{X} \theta ^* + \varepsilon) \\
    &= \Omega^{-1} X^\top ((X + \bar{X} - X) \theta^* + \varepsilon) \\
    &= \Omega ^{-1}  X^\top X \theta^*+ \Omega ^{-1} X^\top (\bar{X} - X) \theta^* + \Omega ^{-1} X^\top \varepsilon  \\
    &= \Omega ^{-1}( \rho I  +  X^\top X ) \theta^* - \rho \Omega ^{-1} \theta^*+ \Omega ^{-1} X^\top (\bar{X} - X) \theta^* + \Omega ^{-1} X^\top \varepsilon \\
    &=  \theta^* - \rho \Omega ^{-1} \theta^*+ \Omega ^{-1} X^\top (\bar{X} - X) \theta^* + \Omega ^{-1} X^\top \varepsilon \\
    \implies x^\top \hat{\theta} - x^\top \theta^*
    &=  \left\langle x, X^\top \varepsilon \right\rangle_{\Omega ^{-1}} + \left\langle x, (X^\top (\bar{X} - X) - \rho I) \theta^*  \right\rangle_{\Omega ^{-1}}\\
    &\leq  \left\|x\right\|_{\Omega^{-1}} \left(\| X^\top \varepsilon \|_{\Omega ^{-1}} + \left\|\theta ^*\right\|_2 \left\|\bar{X} - X\right\| + \rho \| \theta^*\|_{\Omega ^{-1}} \right)  \\
    &\leq  \left\|x\right\|_{\Omega^{-1}} \left(\| X^\top \varepsilon \|_{\Omega ^{-1}} + \left\|\theta ^*\right\|_2 \left\|\bar{X} - X\right\| + \sqrt{\rho} \| \theta^*\|_{2} \right)
\end{align*}

The second inequality uses \(\lambda_{\rm max}(\Omega^{-1}) \geq 1 / \rho\), and the first inequality uses Cauchy-Schwarz and the following bound:

\begin{align*}
    x^\top \Omega ^{-1} X ^\top (\bar{X} - X) \theta^*
    &\leq \left\|X \Omega ^{-1} x\right\|_{2} \left\|(\bar{X} - X) \theta^*\right\|_{2} \\
    &= \sqrt{\left\|X \Omega ^{-1} x\right\|_{2}^{2}} \left\|(\bar{X} - X) \theta^*\right\|_{2} \\
    &\leq \sqrt{\left\|X \Omega ^{-1} x\right\|_{2}^{2} + \rho x^\top \Omega ^{-1} \Omega ^{-1} x } \left\|(\bar{X} - X) \theta^*\right\|_{2} \\
    &= \sqrt{x^\top \Omega ^{-1} (X^\top X + \rho I) \Omega ^{-1} x } \left\|(\bar{X} - X) \theta^*\right\|_{2} \\
    &= \sqrt{x^\top \Omega ^{-1} x } \left\|(\bar{X} - X) \theta^*\right\|_{2} \\
    &= \left\| x\right\|_{\Omega^{-1}} \left\|(\bar{X} - X) \theta^*\right\|_2 \\
    &= \left\| x\right\|_{\Omega^{-1}} \left\|\bar{X} - X \right\| \left\|\theta^*\right\|_2
\end{align*}
    
\item Choosing \(x = \Omega (\hat{\theta} - \theta^*)\), we have:

        \begin{equation*}
          \left\|\hat{\theta} - \theta ^*\right\|_\Omega^{2} \leq \left\|\hat{\theta} - \theta ^*\right\|_\Omega \left(\| X^\top \varepsilon \|_{\Omega ^{-1}} + \left\|\theta^*\right\|_2 \left\|X - \bar{X}\right\| + \sqrt{\rho} \|\theta ^*\|_{2} \right)
        \end{equation*}

\item Theorem 1 of \cite{abbasi-yadkoriImprovedAlgorithmsLinear2011} gives that with probability \(1 - \delta/2\),

        \begin{equation*}
          \left\|X^\top \varepsilon\right\|_{\Omega^{-1}} \leq \sigma \sqrt{d \log \left( \frac{d \rho + t \|X\|_{2, \infty}}{\delta } \right)}
        \end{equation*}

  \item Finally, we need only bound $\|X\|_{2, \infty}$ in terms of $\|\bar{X}\|_{2, \infty}$:

        \begin{equation*}
          \|X\|_{2, \infty} \leq \|\bar{X}\|_{2, \infty} + \|X - \bar{X}\|_{2, \infty} \leq \|\bar{X}\|_{2, \infty} + \|X - \bar{X}\|
        \end{equation*}
        where the first inequality uses the triangle inequality, and the second uses the generic norm inequality $\|A\|_{2, \infty} = \max_{\|u\|_{2} = 1} \|A u \|_{\infty} \leq \max_{\|u\|_{2}= 1} \|A u\|_{2} = \|X\|$

\end{enumerate}
\end{proof}

\subsubsection{Completing the proof of Lemma~\ref{le:clean-est-prob}}
Applying \cref{thm:eiv-ci}, with $\delta=1/t^2, d=r+1, \rho=1$, and using that $\|\bar{X}\|_{2, \infty} \leq \|\bar{Z}\|_{2, \infty} + 1 \leq B + 1$, we immediately have that

\begin{equation*}
\|\hat{\theta}_t - \theta^*\|_{\Omega_t} \leq \sigma \sqrt{2(r+1) \log \left( t\left(r+1 + t (B + 1 + \|\hat{Z}_t - \bar{Z}_t\|) \right) \right)}  +  \left\|\theta ^*\right\|_2 (\|\hat{Z}_t - \bar{Z}_t\| + 1)
\end{equation*}

On $C_t^{\rm latent}$, which occurs with probability $1 - \frac{1}{t^{8}}$, we have the bound $\|\hat{Z}_t - \bar{Z}_t\| \leq 20 \sigma \sqrt{\max(n, T) / n}$. Then via a union bound, with probability $1 - \frac{1}{t^2} - \frac{1}{t^8} \geq 1 - \frac{2}{t^2}$  it holds that

\begin{align*}
\|\hat{\theta}_{t} - \theta^{*}\|_{\Omega_{t}}
  &\leq \sigma \sqrt{2(r+1) \log \left( t\left(r+1 + t (B + 1 + 20 \sigma \sqrt{\max(n, T) / n})\right) \right)} \\
  &\quad+ (\|\lambda^{*}\|+|\tau^{*}|)(20\sigma\sqrt{\max(n,T)/n} + 1)\\
&= \beta_{t}/2 = O(\sqrt{r \log (rt)})
\end{align*}

Further,
\begin{align*}
|\hat{\tau}_{t}-\tau^{*}| = |\langle \hat{\theta}_{t} - \theta^{*}, e_1 \rangle| \leq \|\hat{\theta}_{t} - \theta^{*}\|_{\Omega_{t}} \|e_1\|_{\Omega_{t}^{-1}} = \beta_{t} \hat{\sigma}_{t}/2.
\end{align*}

This completes the proof.

\subsection{Proof of Lemma \ref{le:norm_mismatch}}\label{sec:proof-norm-mismatch}
To begin, we have the following lemma to connect the $\|\cdot\|_{\bar{\Omega}}$ and $\|\cdot\|_{\Omega}$ norms.
\begin{lemma}
For some \(\bar{X}, X \in \mathbb{R}^{t \times r}, \rho > 0\), let \(\bar{\Omega} = \rho I + \bar{X}^\top \bar{X}, \Omega = \rho I + X^\top X\), and \(\Xi = X - \bar{X}\). Then for any vector \(a \in \mathbb{R}^{r}\),  it holds that
\begin{equation*}
\left\|a\right\|_{\bar{\Omega}} \leq \left\|a\right\|_{\Omega} (1 + \rho^{-\frac{1}{2}} \left\|\Xi\right\|)
\end{equation*}

\label{thm:norm-differences}
\end{lemma}

\begin{proof}

 Expanding the definition of \(\Omega\), we have:

\begin{align*}
    \left\|a\right\|_{\bar{\Omega} - \Omega}^2 &= a^\top (\bar{\Omega} - \Omega) a \\
    &= a^\top \left(\Xi^\top X + X^\top \Xi + \Xi ^\top \Xi \right)a \\
    &\leq 2 \left\|a\right\|_2 \left\|\Xi\right\| \left\|X a\right\|_2 + \|a\|_{2}^{2}\left\|\Xi^\top \Xi\right\| \\
    &\leq \frac{2}{\sqrt{\rho}}\|a\|_{\Omega}^{2} \|\Xi\| + \frac{1}{\rho} \|a\|_{\Omega}^{2} \|\Xi\|^{2}
    \end{align*}

where the first inequality follows from Cauchy-Schwarz and the triangle inequality, and the second inequality uses the following inequalities:

\begin{itemize}
\item \(\| X a \|_{2}^{2} = a^\top X^\top X a = a ^\top (X^\top X + \rho I) a - \rho \|a\|_{2}^2 =  \left\|a\right\|_{\Omega}^{2} - \rho \left\|a\right\|_2^{2}  \leq \left\|a\right\|_{\Omega}^2\)
\item \(\left\|a\right\|_{\Omega}^{2} \geq \sigma_{\rm r}(\Omega)  \left\|a\right\|_2^{2} \geq \rho \left\|a\right\|_2^{2}\)
\end{itemize}
We then use this to bound the difference between \(\left\|a\right\|_{\bar{\Omega}}\) and \(\left\|a\right\|_{\Omega}\):

\begin{align*}
    \left\|a\right\|^2_{\bar{\Omega}}
    &= \left\|a\right\|^2_{\Omega} + \left\|a\right\|^2_{\bar{\Omega} - \Omega} \\
    &\leq  \left\|a\right\|^2_{\Omega} + \frac{2}{\rho^{\frac{1}{2}}} \left\|a\right\|^2_{\Omega} \|\Xi\| + \frac{1}{\rho} \left\|a\right\|^2_{\Omega} \left\|\Xi\right\|^2 \\
    &= \left\|a\right\|^2_{\Omega}\left(1 + \frac{2}{\rho^{\frac{1}{2}}}  \|\Xi\|+ \frac{1}{\rho} \left\|\Xi\right\|^2 \right) \\
    &= \left\|a\right\|^2_{\Omega} \left(1 + \frac{1}{\rho^{\frac{1}{2}}}  \|\Xi\| \right)^{2} \\
    \end{align*}
\end{proof}
\cref{thm:norm-differences} also implies the same bound on difference between the inverse norms $\left\|\cdot\right\|_{\Omega ^{-1}}$ and $\left\|\cdot\right\|_{\bar{\Omega}^{-1}}$; i.e.

\begin{align*}
\left\|a\right\|_{\bar{\Omega}^{-1}}  = \max_{\|u\|_{\bar{\Omega}} \leq 1} \langle a, u \rangle \leq \max_{\|u\|_{\Omega} \leq 1+ \rho^{-1/2} \|\Xi\|} \langle a, u \rangle \leq (1+ \rho^{-1/2} \|\Xi\|) \|a\|_{\Omega^{-1}}.
\end{align*}

Then, applying \cref{thm:norm-differences} and the rotation invariance discussed in \cref{sec:invariance-rotation}, under $C^{\rm latent}$, we have
\begin{align*}
\left\|a\right\|_{\bar{\Omega}^{-1}_{t}} \leq (1+ \rho^{-1/2} \|\Xi\|) \|a\|_{\Omega^{-1}_t} \leq (1+\|\hat{Z}_{t} - \bar{Z}_{t}\|)\|a\|_{\Omega^{-1}_t}  \leq (1+\alpha) \|a\|_{\Omega^{-1}_t}.
\end{align*}
with $\rho \triangleq 1.$ This completes the proof.

\section{Proofs of Inferential Results}

\subsection{Proof of Proposition \ref{prop:synthetic-control-guarantee}}

\begin{proof}
By definition, the estimator $\tauSC$ is
\begin{align}
    \tauSC 
    &:= \frac{1}{T}\sum_{t=1}^{T} \left(y_{t}^{0} - \sum_{i=1}^{n} w_{i} y_{t}^{i}\right) \nonumber\\
    &= \tau^{*} + \frac{1}{T} \sum_{t=1}^{T} \left(\lambda^{*\top} \bar{z}_{t} - \sum_{i=1}^{n}w_i \lambda^{i \top} \bar{z}_{t}\right) + \underbrace{\frac{1}{T} \sum_{t=1}^{T} \left(\epsilon_{t}^{0} - \sum_{i=1}^{n} w_i \epsilon_{t}^{i}\right)}_{R_1} \nonumber\\
    &= \tau^{*} + \left(\lambda^{*\top} - \sum_{i=1}^{n} w_i \lambda^{i\top}\right) \left(\frac{1}{T} \sum_{t=1}^{T} \bar{z}_t\right) + R_1.\label{eq:tau-taustar}
\end{align}
Let $Z = [\bar{z}_1^{\top}; \bar{z}_2^{\top}; \cdot; \bar{z}_{T_0}^{\top}] \in \R^{T_0 \times r}.$ Let $E \in \R^{(n+1)\times T_0}$ be noise matrix with entries $E_{ij} = \epsilon^{i}_{j-T_0}$ for $i = 0, 1, \dotsc, n, j = 1, 2, \dotsc, T_0.$ Let $E_{i}$ be the i-th row of $E.$

Then, from $y_{t}^{0} = \sum_{i=1}^{n}w_i y_{t}^{i}, t = -T_0+1, \dotsc, -1, 0$, we have
\begin{align*}
    Z \lambda^{*} + E_{0} = Z \left(\sum_{i=1}^{n} w_i \lambda^{i}\right) + \sum_{i=1}^{n} w_i E_{i}   
\end{align*}
Let $Z^{-1}$ be the pseudo-inverse of $Z$. We then have $Z^{-1}Z = I_{r}$ and
\begin{align}
    \lambda^{*} - \left(\sum_{i=1}^{n} w_i \lambda^{i}\right) = Z^{-1} \left(\sum_{i=1}^{n} w_i E_{i} - E_{0}\right). \label{eq:u1-bound}
\end{align}
Plugging \cref{eq:u1-bound} back into \cref{eq:tau-taustar}, we have
\begin{align}
     \tauSC  - \tau^{*} 
     &=\left(Z^{-1} \left(\sum_{i=1}^{n} w_i E_i - E_0\right)\right)^{\top} \left(\frac{1}{T} \sum_{t=1}^{T} \bar{z}_t\right) + R_1\\
     &= \underbrace{E_0 Z^{-1\top} \left(\frac{1}{T} \sum_{t=1}^{T} \bar{z}_t\right)}_{R_2} + \underbrace{w^{\top} E_{1:n} Z^{-1\top} \left(\frac{1}{T} \sum_{t=1}^{T} \bar{z}_t\right)}_{R_3} + R_1
\end{align}
where $E_{1:n}$ is the sub-matrix of $E$ consisting of rows indexed by $1,2, \dotsc, n.$

For $R_1$, note that $w$ and $\epsilon^{i}_{t}$ for $t>T_0$ are independent, then
\begin{align}
    R_1 \overset{dist}{=} \mathcal{N}\left(0, \sigma^2 \left(1+\sum_{i=1}^{n}w_i^2\right)\frac{1}{T}\right)
\end{align}
Note that $\sum_{i=1}^{n} w_i^2 \leq 1$ since $w_i \geq 0$ and $\sum_{i} w_i \leq 1$. Then with probability $1-\delta$, 
$$
|R_1| \lesssim \frac{\sigma}{\sqrt{T}}\sqrt{\log(1/\delta)}.
$$

For $R_2$, let $a := Z^{-1\top} \left(\frac{1}{T} \sum_{t=1}^{T} \bar{z}_t\right).$ Note that
\begin{align*}
    \|a\| \leq \frac{\norm{\frac{1}{T} \sum_{t=1}^{T} \bar{z}_t}}{\sigma_{r}(Z)} \leq \frac{c_2}{\sqrt{c_1 T_0}}.
\end{align*}
Since $E_{0}$ are i.i.d Gaussian and independent from $a$, with probability $1-\delta$, we have
\begin{align*}
    |R_2| \lesssim \frac{c_2 \sigma}{c_1 \sqrt{T_0}} \sqrt{\log(1/\delta)}.
\end{align*}
For $R_3$, the difficulty is that $E$ and $w$ are not independent. Note that 
\begin{align*}
    |R_3| = |w^{\top} E_{1:n} a| \leq \norm{E_{1:n} a}_{\infty} = \max_{i \in [1:n]} |E_i^{\top} a|.
\end{align*}
By the union bound and $E_{i}^{\top} a$ being Gaussian, with probability $1-\delta$, we have
\begin{align*}
    |R_3| \lesssim \frac{c_2 \sigma}{\sqrt{c_1T_0}} \sqrt{\log(1/\delta) + \log(n)}.
\end{align*}
This completes the proof. 
\end{proof}

\subsection{Proof of Proposition \ref{prop:SCTS-full-inference-result}}
We will prove a generalized result of \cref{prop:SCTS-full-inference-result} below. 
\begin{proposition}
The followings hold. 
\begin{itemize}
\item[(a)]  With probability $1-O(\delta)$, if $M > T / 2$, then
\begin{align*}
|\tauTS - \tau^{*}| \lesssim \frac{\sigma}{\sqrt{T}}\sqrt{\log(1/\delta)} + \frac{c_2\sigma}{c_1 \sqrt{T_0}} \sqrt{\log(1/\delta) + \log(n)}. 
\end{align*}
\end{itemize}
Furthermore, 
\begin{itemize}
\item[(b)] When $\tau^{*} \geq 0$, with probability at most $\frac{2R(T)}{T \tau^{*}}$, $M \leq T/2$ (i.e., $\tauTS = 0$). 
\item[(c)] When $\tau^{*} < 0$, with probability at least $1 - \frac{2R(T)}{T|\tau^{*}|}$, $M \leq T/2$ (i.e., $\tauTS = 0$). \end{itemize}
\end{proposition}
\begin{proof}
In order to show (a), we follow the same analysis as the proof of \cref{prop:synthetic-control-guarantee}, where the only term that needs to be re-analyzed is 
\begin{align*}
    R_1:=\frac{1}{M} \sum_{t=1}^{T} \mathbf{1}\{a_t=1\}\left(\epsilon^{0}_{t} - \sum_{i=1}^{n} w_i \epsilon^{i}_{t}\right)
\end{align*}
Let $z_t := \mathbf{1}\{a_t=1\}$. Then
\begin{align}
    R_1 = \frac{1}{M} \sum_{i=1}^{T} \epsilon^{0}_{t} z_t - \frac{1}{M} \sum_{t=1}^{T} \sum_{i=1}^{n} w_i \epsilon^{i}_{t} z_t
\end{align}
Note that $z_t$ is independent from the future $\epsilon^{j}_{k}$ for $k\geq t$. Hence we can use Azuma's inequality (a sub-Gaussian variant \cite{shamir2011variant}) to obtain, with probability $1-\delta$,
\begin{align}
    \left|\sum_{t=1}^{T} \epsilon^{0}_{t} z_t\right| \lesssim \sqrt{T} \log(1/\delta) \sigma. 
\end{align}
Similar bounds can be obtained for $\sum_{t=1}^{T} \left(\sum_{i=1}^{n} w_i \epsilon^{i}_{t}\right) z_t.$ This provides, with probability $1-O(\delta)$, 
\begin{align*}
    R_1 \lesssim \frac{\sqrt{T}}{M} \log(1/\delta) \sigma.
\end{align*}
Using $M > \frac{T}{2}$, we finish the proof for (a).

Next, consider the case $\tau^{*} \geq 0.$ The regret incurred for each instance is $R = (T-M)\tau^{*}.$ By Markov inequality, we then have
\begin{align*}
\Pr\left((T-M)\tau^{*} \geq \frac{T \tau^{*}}{2}\right) \leq \frac{R(T)}{T \tau^{*} / 2}.
\end{align*}
Equivalently, 
\begin{align*}
\Pr \left(M \leq \frac{T}{2}\right) \leq \frac{2 R(T)}{T \tau^{*}}.
\end{align*}
Hence, with probability at most $\frac{2R(T)}{T \tau^{*}}$, we have $M \leq T/2$, i.e., (b) holds. 

Finally, consider the case $\tau^{*} < 0.$ The regret incurred is $R = M|\tau^{*}|$. Then, by Markov inequality, 
\begin{align*}
\Pr\left(M |\tau^{*}| > \frac{T}{2} |\tau^{*}| \right) \leq \frac{R(T)}{|\tau^{*}| T / 2}.
\end{align*}
This proves (c). 
\end{proof}

\section{Technical Lemmas}

\begin{theorem}Let \(M, \bar{M}, E \in \R^{n\times m}\) be arbitrary matrices such that \(M = E + \bar{M}\), and let \(\bar{U} \bar{S} \bar{V}^\top, U S V^\top\) be the SVDs of \(\bar{M}, M\) respectively, truncated to \(r\) singular values.
Assume \(\sigma_{r}( \bar{M} ) > 0\) and \(\sigma_{r+1}( \bar{M} ) = 0\). Then, there exists an orthogonal matrix \(H\) such that
\begin{equation*}
\| \bar{V}\bar{S}H - VS\| \leq 4\|E\|.
\end{equation*}
\label{thm:rotation-ZPhi}
\end{theorem}
\begin{proof}
Let $\bar{U}^{\top}U = \tilde{U}\tilde{\Sigma}\tilde{V}^{\top}$ be the SVD of $\bar{U}^{\top}U.$ We construct $H \triangleq \tilde{U}\tilde{V}^{\top}.$ Then
\begin{align*}
\| \bar{V}\bar{S}H - VS\| 
&\overset{(i)}{\leq} \| \bar{V}\bar{S}\bar{U}^{\top}U - VS\| + \| \bar{V}\bar{S}(H - \bar{U}^{\top}U)\| \\
&\overset{(ii)}{\leq} \|(\bar{V}\bar{S}\bar{U}^{\top} - VSU^{\top})U\| + \| \bar{V} \| \|\bar{S}(H - \bar{U}^{\top}U)\| \\
&\leq  \|\bar{V}\bar{S}\bar{U}^{\top} - VSU^{\top}\| +  \|\bar{S}(H - \bar{U}^{\top}U)\| 
\end{align*}
where (i) is due to the triangle inequality and (ii) is due to $U^{\top}U = I_{r}.$ Note that
\begin{align*}
\|\bar{V}\bar{S}\bar{U}^{\top} - VSU^{\top}\| 
&= \|\bar{M}^{\top} - M^{\top} + M^{\top} - VSU^{\top}\|\\
&\leq \|E\| + \|M^\top - V S U ^\top\|\\
&\leq \|E\| + \sigma_{r+1}(M^{\top})\\
&\leq 2\|E\|.
\end{align*}
Here, the last inequality uses Weyl's inequality $\sigma_{r+1}(M^{\top}) \leq \sigma_{r+1}(\bar{M}) + \|E\| \leq \|E\|.$

It remains only to bound
\begin{align*}
\|\bar{S}(H - \bar{U}^{\top}U)\|  
&= \|\bar{S}\tilde{U} (I_{r} - \tilde{\Sigma})\tilde{V}^{\top}\| \\
&= \sqrt{\|\bar{S}\tilde{U} (I_{r} - \tilde{\Sigma})^2 \tilde{U}^{\top} \bar{S}^{\top}\|}.
\end{align*}
Note that for any $a, b \in \R^{r}$ with $\|a\|=\|b\|=1$, we have $a^{\top}\bar{U}^{\top}Ub \leq \|a\|\|b\| = 1.$ Therefore $0\leq \tilde{\Sigma}_{ii} \leq 1$ for $i\in [r]$. This implies $(I_{r} - \tilde{\Sigma})^2 \preceq (I_{r} - \tilde{\Sigma}^2)^2$ and further 
$$
\bar{S}\tilde{U} (I_{r} - \tilde{\Sigma})^2 \tilde{U}^{\top} \bar{S}^{\top} \preceq \bar{S}\tilde{U} (I_{r} - \tilde{\Sigma}^2)^2 \tilde{U}^{\top} \bar{S}^{\top}
$$
Then we have
\begin{align*}
\sqrt{\|\bar{S}\tilde{U} (I_{r} - \tilde{\Sigma})^2 \tilde{U}^{\top} \bar{S}^{\top}\|} \preceq \sqrt{\|\bar{S}\tilde{U} (I_{r} - \tilde{\Sigma}^2)^2 \tilde{U}^{\top} \bar{S}^{\top}\|} = \|\bar{S}\tilde{U} (I_{r} - \tilde{\Sigma}^2) \tilde{U}^{\top}\|.
\end{align*}
Here the last equality uses $\|AA^{\top}\| = \|A\|^{2}.$ Hence,

\begin{align*}
\|\bar{S}(H - \bar{U}^{\top}U)\|  
&\leq  \|\bar{S}\tilde{U} (I_{r} - \tilde{\Sigma}^2) \tilde{U}^{\top}\|\\
&\overset{(i)}{=} \|\bar{S} \bar{U}^{\top} (I_{n} - UU^{\top})\bar{U} \| \\
&= \|\bar{V}^{\top}\bar{M}^{\top} (I_{n} - UU^{\top}) \bar{U}\|\\
&\leq \|\bar{M}^{\top} (I_{n} - UU^{\top})\| \\
&= \|(M-E)^{\top}(I_{n} - UU^{\top}) \|\\
&\leq \|E\| + \sigma_{r+1}(M)\\
&\leq 2\|E\|.
\end{align*}

Here (i) uses $I_{r} = \bar{U}^{\top}\bar{U} = \tilde{U} \tilde{U}^\top$ (since $\tilde{U}$ is an $r \times r$ rotation matrix) and $(\bar{U}^{\top} U)(U^{\top} \bar{U}) = \tilde{U}\tilde{\Sigma}^2\tilde{U}.$ This completes the proof.
\end{proof}

\begin{theorem}[Matrix concentration] Suppose that \(E_{ij} \overset{\text{iid}}{\sim} \mathcal{N}(0, \sigma^2)\), and let \(E \in \mathbb{R}^{n \times t}\). Then, there exists some universal constant \(c_3\) such that, with probability \(1 - \frac{1}{(n \lor t)^{8}}\) , we have \(\left\|E\right\| \leq 4 \sigma \sqrt{ n \lor t} + c_{3} \sigma \log(n \lor t)\)
\label{thm:matrix-concentration}
\end{theorem}
For sufficiently large \(n \lor t\)  this can be further bounded as \(\left\|E\right\|\leq 5\sigma \sqrt{n \lor t}\). See e.g. \cite{chen2020spectral} Theorem  3.1.4 for reference.

\begin{lemma}
  \label{le:epl-full}
  Consider the a sequence $\{ x_t \}$ where $x_t \in \mathbb{R}^{d}$ and  $\|x_t\|_{2} \leq B \; \forall t$, and define $\Omega_t = \rho I + \sum_{t=1}^T x_t x_t^\top$. Then, it holds that
\begin{equation*}
  \sum_{t=1}^{T} \left\|x_t\right\|_{\Omega_{t-1} ^{-1}} \leq \sqrt{\frac{B^2}{\rho} d T \log \left(1 + \frac{T B^2}{d \rho}\right)}
\end{equation*}
\end{lemma}

See e.g. \cite{lattimoreBanditAlgorithms2020} for proof.

\end{document}